\documentclass{tlp}
\usepackage{hyperref}
\usepackage{stmaryrd}
\usepackage{url}
\usepackage{amssymb}
\usepackage{dsfont}
\usepackage{aopmath}
\usepackage{graphicx}
\usepackage{type1cm}
\usepackage{eso-pic}

\newcommand{\GLop}{\gamma}
\newcommand{\dlGLop}{\alpha}

\newcommand{\tva}[1]{\mathcal{I}_{#1}}
\newcommand{\neck}{\leftarrow}
\newcommand{\naf}{\sim}

\newcommand{\bclos}{\mathbf{\beta}}

\newcommand{\proofend}{$\square$}

\def\dsmodels{\mathrel|\joinrel\approx}
\def\dsrefutes{\mathrel\approx\joinrel\mid}

\def\asmodels{\mathrel|\joinrel\approx_{ADL}}
\def\asrefutes{\mathrel\approx\joinrel\mid_{ADL}}
\def\dmodels{\mathrel|\joinrel\sim}
\def\drefutes{\mathrel\sim\joinrel\mid}

\newcommand{\itrue}[1]{\mathcal{T}_{#1}}
\newcommand{\ifalse}[1]{\mathcal{U}_{#1}}

\newcommand{\ordC}{\gamma}
\newcommand{\ordG}{\eta}
\newcommand{\ordJ}{\kappa}
\newcommand{\ordK}{\lambda}

\newcommand{\myemptyset}{\varnothing}

\newcommand{\defrule}{\Rightarrow}
\newcommand{\srule}{\rightarrow}
\newcommand{\defeater}{\rightsquigarrow}

\newcommand{\jconsA}{T}
\newcommand{\junfA}{U}
\newcommand{\jwfA}{W}

\newcommand{\bdldefeasibly}{+\delta}
\newcommand{\bdldefeasiblyno}{-\delta}

\def\definitely{+\Delta}

\def\definitelyno{-\Delta}

\newtheorem{definition}{Definition} % [section]
\newtheorem{example}{Example} % [section]

\begin{document}
\bibliographystyle{acmtrans}

 \submitted{September 10 2010}
 \revised{May 21 2011}
 \accepted{June 1 2011}

\title[Interdefinability of Defeasible Logic and Logic Programming under the WFS]{Interdefinability of Defeasible Logic and Logic Programming under the Well-Founded Semantics}

\author[Frederick Maier]
{FREDERICK MAIER\thanks{\textbf{To appear in Theory and Practice of Logic Programming (TPLP).}}
\thanks{Portions of this work were performed while the author was a doctoral student at The University of Georgia in Athens, Georgia.}
\\
Kno.e.sis Center, Wright State University
}

%\pagerange{\pageref{firstpage}--\pageref{lastpage}}
%\volume{\textbf{10} (3):}
%\jdate{March 2009}
%\setcounter{page}{1}
%\pubyear{2009}

\maketitle

\label{firstpage}

\begin{abstract}
We provide a method of translating theories of Nute's defeasible logic into logic programs, and a corresponding translation in the opposite direction. Under certain natural restrictions, the conclusions of defeasible theories under the  ambiguity propagating defeasible logic ADL correspond to those of the well-founded semantics for normal logic programs, and so it turns out that the two formalisms are closely related.  Using the same translation of logic programs into defeasible theories, the semantics for the ambiguity blocking defeasible logic NDL  can be seen as indirectly providing an ambiguity blocking semantics for logic programs.

We also provide antimonotone operators for both ADL and  NDL, each based on the Gelfond-Lifschitz (GL) operator for logic programs. For defeasible theories without defeaters or priorities on rules, the operator for ADL corresponds to the GL operator and so can be seen as partially capturing the consequences according to ADL. Similarly, the operator for NDL captures the consequences according to NDL, though in this case no restrictions on theories apply. Both operators can be used to define stable model semantics for defeasible theories.

\end{abstract}


\section{Introduction}

Defeasible logic is a family of rule-based nonmonotonic reasoning formalisms originally developed by Donald Nute (Nute \citeyearNP{nute1986,nute1994,nute1997}; Nute et al. \citeyearNP{nute1989}). Over the years, many variants have been proposed,
 with the most recent system created by Nute himself---an ambiguity blocking logic which we call NDL---appearing in the late 90s
  (Nute \citeyearNP{nute1999,nute2001}; Donnelly \citeyearNP{donnelly}). An ambiguity propagating counterpart to NDL, called ADL, was developed considerably later \cite{maier2006}.  Working separately, David Billington \citeyear{billington1993} presented a
quantified version of one of Nute's logics and showed it to be cumulative. Billington, together with Grigoris Antoniou,
Michael Maher, Guido Governatori, and others, would later go on to
publish a number of papers on this logic and its offshoots
(Antoniou et al.  \citeyearNP{antoniou2000,antoniou1,antoniou2001,antoniou05}; Governatori et al. \citeyearNP{governatori2004}; Maher and Governatori \citeyearNP{maher99}; Maher et al. \citeyearNP{maher01b}).
This logic---which we call BDL---and its variants are the ones most frequently encountered in the literature.

Defeasible logic possesses several virtues which warrant its study. It is of low computational complexity  compared to, say, default logic
 \cite{reiter:1} or logic programming under the stable model semantics \cite{gelfond88}.  E.g., some of the logics based on BDL have linear complexity \cite{maher01b}.
Also, the different variants of defeasible logic express a variety of different intuitions, and so, as Antoniou et al. \citeyear{antoniou2000} have said, defeasible logic forms a ``flexible framework'' for knowledge representation.
Furthermore, the device primarily responsible for making defeasible logic defeasible---namely, the defeasible rule---is intuitively easy to grasp, arguably easier than the {default negation} $\sim$ used in logic programming. E.g., we at least  find the defeasible rule

\begin{center}$\{bird(X)\} \Rightarrow fly(X)$\end{center}

\noindent{}(which might be read as ``Birds usually can fly'') to be more understandable than its logic program counterpart:

\begin{center}$fly(X) \neck bird(X), \naf \neg fly(X)$\end{center}

\noindent{}Such considerations are important when it comes to creating and maintaining knowledge--based systems.

Defeasible logic is nevertheless relatively little known in the nonmonotonic reasoning (NMR) community, and relationships to more mainstream NMR formalisms have been only partially studied.
There are exceptions.   Antoniou and Billington \citeyear{antoniou2001b} have shown how an ambiguity propagating variant of BDL can be embedded into default logic, and Brewka \citeyear{brewka} provided a simple scheme for translating the same logic (though without ``team defeat'') into logic programming under his own prioritized well-founded semantics \cite{brewka96}. Later, Antoniou et al. \citeyear{antoniou05} provided an alternative embedding of defeasible theories into logic programs and  showed a relationship between the BDL-conclusions of theories and both the Kunen \citeyear{kunen} and stable model semantics of their embeddings.

Similar analyses have not been performed for NDL and ADL, however. NDL was not well known at the time Brewka's paper was written, and ADL did not exist until 2006.  This is unfortunate, since these logics incorporate features not found in other defeasible logics.  Particularly, NDL and ADL include \emph{failure-by-looping}, a mechanism  to weed out circular arguments.
Its absence  contributes greatly to the low complexity of the other logics, but it also means that the logics fail to draw
reasonable conclusions in some cases.  E.g., if given the single rule
\begin{center}
$\{p\} \rightarrow p$
\end{center}
and nothing else, then the earlier defeasible logics would be unable to conclude anything at all about $p$; it is neither  provable nor refutable in these logics. However, based on the rule alone, there's no reason to believe $p$, and so it should be unprovable. Both NDL and ADL are able to show this.

Failure-by-looping is conceptually similar to the notion of an \emph{unfounded set} in the \emph{well-founded semantics} (WFS) for logic programs \cite{vangelder91}. Indeed, a recognition of this is what led to the development of adequate semantics \cite{maier2009} for both NDL and ADL---semantics which are based explicitly on the WFS. Historically, defeasible logics have been defined  proof-theoretically, with semantics coming only later. Prior to 2006, the only semantics offered for NDL \cite{donnelly} was sound but incomplete, even for finite propositional theories.

Given the similarity of the semantics for NDL and ADL to the WFS for logic programs, it is natural to inquire whether each formalism can be translated into the other. For proponents of defeasible logic, the benefit of interdefinability with logic programs would be that preexisting logic program reasoners---such as XSB \cite{xsb} or \texttt{smodels} \cite{niemela2001}---could be used to draw conclusions according to defeasible logic, and many of the theoretical results already known about logic programs could be applied to defeasible theories. For proponents of logic programming, the benefit of interdefinability would be that results known about defeasible logic could be applied to logic programs. Furthermore, interdefinability would allow certain programs to be represented in a more concise and intuitively acceptable manner.

The present paper takes up interdefinability and related issues. It is shown here that for a restricted class of defeasible theories (those with minimal conflict sets, no defeaters, and no priorities on rules) the semantics for ADL corresponds to the WFS for normal logic programs. That is, there exists a
rather natural translation of a defeasible theory into a logic program (under the WFS) that preserves the ADL-consequences of
the theory. It is also shown that a consequence preserving translation exists in the other direction. The closed-world assumption of the WFS is easily represented as a set of defeasible rules with empty bodies.  And so,  given the restrictions, either formalism can be embedded into the other.

We also define antimonotone operators $\alpha$ and $\beta$ for defeasible theories (with $\alpha$ propagating ambiguity and $\beta$ blocking it) and show that when unprioritized defeasible theories are translated into logic programs, $\alpha$ coincides with the Gelfond--Lifschitz operator $\gamma$ \cite{gelfond88}. Given this and the correspondence between ADL and the WFS, it immediately follows that an alternating fixpoint procedure based on $\alpha$ can be used on unprioritized theories to generate the consequences according to ADL.  This parallels the known relationship between $\gamma$ and the WFS \cite{baral91}. Additionally, a similar fixpoint procedure,  based on $\beta$, exactly captures the consequences according to NDL.  Both operators can be used to define stable model semantics for defeasible theories.

Given that NDL blocks ambiguity and both ADL and the WFS propagate it, it is no surprise that there is no correspondence between the NDL-consequences of a defeasible theory and the well-founded model of its logic program counterpart.  In general, the consequences according to the WFS and ADL are a subset of those according to NDL. However,  using the same scheme to translate logic programs into defeasible theories, NDL can be
viewed as indirectly providing an ambiguity blocking semantics for logic programs.

The remainder of the paper is organized as follows:
Sections \ref{wfs}, \ref{DL:Syntax} and \ref{WFSSec} give overviews, respectively, of logic programming with the WFS, the syntax of defeasible logic (in general), and the semantics of NDL and ADL (in particular).
The consequence-preserving translations of defeasible theories into logic programs (and vice versa) are given in
Sections \ref{DL2LP} and \ref{LPtoDL}.
Section \ref{AltFixpoints} presents the operators $\alpha$ and $\beta$. The correspondence between $\alpha$ and $\gamma$ is proven, as is the relationship between $\beta$ and the semantics for NDL. The operators are used in Section \ref{stableModelsForDL} to define semantics for defeasible theories akin to the stable model semantics.
We conclude with a brief discussion of related work, discussing in particular the differences between ADL/NDL and other versions of defeasible logic.

Two appendices are also included. The first presents the proof systems for NDL and ADL. The  second shows how  defeasible theories can be transformed into equivalent ones lacking defeaters, extended conflict sets, and priorities on rules. Given this, it follows that  $\alpha$ and logic programs under the WFS can both be used to compute all of the ADL consequences of defeasible theories.

\section{The WFS for normal logic programs}\label{wfs}

A normal logic program $\Pi$ consists of  rules of the  form

\noindent{}\begin{center}$a \neck b_1, b_2, \ldots, b_n, \naf c_1, \naf c_2, \ldots, \naf c_m$\end{center}

\noindent where  $n$ and $m$ are nonnegative integers, $a$ and each $b_i$ and $c_i$ are first order atomic formulas, and $\sim$ is \emph{default negation}; each $c_i$ is a \emph{default literal}.
In the original WFS, an interpretation $\tva{}$ of program $\Pi$ may be represented as a tuple $\langle \itrue{},\ifalse{}\rangle$, where  $\itrue{}$ and $\ifalse{}$ are disjoint sets of atoms.  Interpretations are thus 3-valued;
an atom $p$ is \emph{true} if it is in $\itrue{}$, \emph{false} if it is in $\ifalse{}$, and  \emph{undefined} otherwise.  Interpretations in defeasible logic are similar, but in defeasible logic   $\itrue{}$ and $\ifalse{}$  can contain a mix of atoms \emph{and} their negations (and so it is somewhat awkward to speak of truth and falsity). To avoid confusion,
 we will say that elements of $\itrue{}$ are \emph{well-founded} and those of $\ifalse{}$ are \emph{unfounded}. What remains is said to be \emph{ambiguous}.  The WFS selects one interpretation to serve as the canonical model (the well-founded model) of the program. Every program is guaranteed to have exactly one well-founded model.

  In the following discussion, we assume that logic programs contain only ground terms, and that the number of rules in the program is countable. $At(\Pi)$ is the set of ground atoms associated with a program, while $Lit(\Pi)$ is defined to be $At(\Pi) \cup \{\neg p| p \in At(\Pi)\}$. The  literal $\neg p$ is the classical complement of atom $p$.
    When it is necessary to talk about a literal and its complement in a general way, we will use of the following notation:
    If $p$ is an atom $a$, then $\overline{p}$ is $\neg a$. If $p$ is $\neg a$, where $a$ is an atom, then $\overline{p}$ is $a$.

  The expressions \emph{head(r)} and \emph{body(r)} refer, respectively, to the head and body of the rule $r$, and $body(r)^{-}$ and $body(r)^{+}$ refer, respectively, to the default and non-default literals of $body(r)$. Similarly, $r^{+}$ is $r$ with all default literals removed.
 An \emph{NAF-free} (or \emph{definite}) logic program is a normal logic program containing no default literals.

If $\itrue{}$ and $\ifalse{}$ are allowed to overlap in interpretations, then the set of all interpretations forms a complete lattice under the relation $\sqsubseteq$, where

\begin{center}
$\langle \itrue{1},\ifalse{1}\rangle \sqsubseteq \langle \itrue{2},\ifalse{2}\rangle$ \emph{iff} $\itrue{1} \subseteq \itrue{2}$ and $\ifalse{1} \subseteq \ifalse{2}$.
\end{center}

\noindent{}This is the so-called knowledge ordering.
The bottom $\bot$ of the lattice is
 $\langle \myemptyset, \myemptyset\rangle$ and the top $\top$  is $\langle At(\Pi),At(\Pi)\rangle$.
The well-founded model of $\Pi$ is defined using the operators $U_\Pi$, $T_\Pi$, and $W_\Pi$, all of which are  monotone on the lattice. $T_\Pi$ is the \emph{immediate consequence operator}.

\begin{center}
$\jconsA_{\Pi}(\tva{}) = \{head(r)$ $|$ $r\in\Pi$, $body(r)^{+} \subseteq \itrue{}$, and $body(r)^{-} \subseteq \ifalse{}\}$.
\end{center}

\noindent{}$U_\Pi$ and $W_\Pi$ are defined via \emph{unfounded sets}, which intuitively are sets for which no external support exists.
 If $\Pi$ is a normal logic program and $\tva{} = \langle \itrue{}, \ifalse{} \rangle$ an interpretation, then a  set  $S \subseteq At(\Pi)$ is an \emph{unfounded set}
of $\Pi$ \emph{wrt} interpretation $\tva{}$ \emph{iff} for each $p \in S$ and each rule $r \in \Pi$ with head $p$, either:

\begin{enumerate}
\item there is a $q \in body(r)^{+}$ such that $q \in
\ifalse{} \cup S$, or
\item there is a $q \in body(r)^{-}$ such that $q \in
\itrue{}$.
\end{enumerate}

\noindent{}Unfounded sets are closed under union. $U_\Pi(\tva{})$ is the \emph{greatest} unfounded set of $\Pi$ \emph{wrt} $\tva{}$:

\begin{center}
$\junfA_{\Pi}(\tva{}) = \bigcup\{A$ $|$ $A$ is an unfounded set of
$\Pi$ with respect to $\tva{}\}$.
\end{center}

$U_\Pi(\tva{})$ and $T_\Pi(\tva{})$ are disjoint, and $W_\Pi$ combines them to form a new interpretation:

\begin{center}
$\jwfA_{\Pi}(\tva{}) = \langle \jconsA_{\Pi}(\tva{}) ,
\junfA_{\Pi}(\tva{}) \rangle$
\end{center}

\noindent Beginning with $\bot$, the following sequence $(\tva{0},\tva{1},\ldots)$ is defined using  $\jwfA_{\Pi}$.

\begin{enumerate}
    \item $\tva{0} = \jwfA_{\Pi}\uparrow 0 = \langle\myemptyset,\myemptyset\rangle$
    \item $\tva{\alpha+1} = \jwfA_{\Pi}\uparrow{}\alpha+1 = \jwfA_{\Pi}(\tva{\alpha})$ (for successor ordinals)
    \item $\tva{\alpha} = \jwfA_{\Pi}\uparrow{}\alpha = \langle \displaystyle\bigcup_{\beta < \alpha} \itrue{\beta}, \displaystyle\bigcup_{\beta < \alpha} \ifalse{\beta}\rangle$ (for limit ordinals)

\end{enumerate}

\noindent The well-founded model $wfm(\Pi)$ of $\Pi$  is defined to be $\jwfA_\Pi\uparrow \lambda$, where $\lambda$ is the closure ordinal of the sequence---i.e. the least $\lambda$ such that $\jwfA_\Pi\uparrow \lambda = \jwfA_\Pi\uparrow (\lambda+1)$.
Since $W_\Pi$ is monotone on the lattice of interpretations, then by the
Knaster-Tarski Theorem \cite{tarski55}, least $lfp(W_\Pi)$ and greatest $gfp(W_\Pi)$ fixpoints of $W_\Pi$
exist. The well-founded model may equivalently be defined as  $lfp(W_\Pi)$.

\begin{example}
\begin{enumerate}
\item $p \neck \naf q_0$
\item $q_n \neck  q_{n+1}$ \emph{(for all $n \in \mathds{N}_0$, where $\mathds{N}_0$ is the set of nonnegative integers.)}
\end{enumerate}
\end{example}

The well-founded model of the above infinite program is $\langle \{p\},\{q_n|n \in \mathds{N}_0\}\rangle$.  Clearly, each $q_i$ is intuitively unfounded. In fact, $\ifalse{1} = \{q_n|n \in \mathds{N}_0\}$.  Given this, $p \in
T_\Pi(\tva{1})$ (in other words, $p \in \itrue{2}$).  The closure ordinal of the sequence is 2.

\section{Defeasible logic}\label{DL:Syntax}
Like logic programs, defeasible logic deals with sets of rules,  where the rules are composed of sets of ground literals (atoms and their classical complements). Unlike logic programs, however, defeasible logic allows three sorts of rules.  If $S$ is a finite set of literals and $p$ is a literal,
then  $S \srule p$ is a {\it strict} rule, $S \defrule p$ a {\it
defeasible} rule, and $S \defeater p$ an {\it (undercutting)
defeater}.
We may read $S \srule p$ as saying ``If $S$, then {\it
definitely} $p$,'' $S \defrule p$ as ``If $S$, then {\it
defeasibly} ({\it normally}, {\it apparently}, {\it
evidently}) $p$'', and $S \defeater p$ as ``If S, then {\it
maybe} $p$. Strict rules with empty bodies are called {\it facts}
and defeasible rules with empty bodies are called {\it
presumptions.} The rule $\myemptyset \defrule p$ may be read as saying ``Presumably, $p$.''

Defeasible logic is intended to be a logic of justification, and in our view, such a logic must be nonmonotonic. From an intuitive standpoint, it is possible for a belief to be justified and nevertheless false, and it is also possible for a belief that was once justified to lose its justification---not because its support has been rejected, but because new information has come to light which either contradicts the belief directly or else undermines it by contradicting its support.

This basic intuition is captured in defeasible logic by allowing {defeasible} rules to be {\it defeated} by other rules. Given the rules

\begin{enumerate}
\item $\myemptyset \srule lives\_alone$
\item $\myemptyset \srule has\_a\_wife$
\item $\{lives\_alone\} \defrule \neg married$
\item $\{has\_a\_wife\} \srule married$
\end{enumerate}

\noindent{}it is reasonable to conclude $lives\_alone$ and $has\_a\_wife$, as both are facts.  However, we cannot,  on pain of contradiction,
simultaneously detach the heads of \emph{both} of the latter two rules.
In defeasible logic, to detach the head of rule 3, we must first show that rule 4 \emph{cannot} be applied. This, in fact, we cannot do. Though the different defeasible logics might formalize the intuition differently, rule 3 is defeasible and rule 4 is strict, and rule 4 defeats rule 3.  In each logic, one can conclude $married$ but not $\neg married$.

The heads of defeaters can never be detached---defeaters exist solely to prevent the application of a conflicting defeasible rule. For example, the defeater $\neg$ \emph{has-intact-flight-feathers} $\rightsquigarrow \neg flies$ might be used to prevent a proof of $flies$ from $\{bird\} \Rightarrow flies$, but it cannot be used to directly prove $\neg flies$.

Strict, defeasible, and defeater rules are collected into \emph{defeasible theories}. Below, if $D$ is a defeasible theory, then $At(D)$ and $Lit(D)$ are defined as they are in logic programming: $At(D)$ is the set of atoms associated with $D$, while $Lit(D)$ is the set of literals.

\begin{definition} A \emph{defeasible theory} $D$ is a triple
$\langle R, C,\prec \rangle$, where $R$ is a countable set of rules, $C$ is a countable set of finite sets of literals in $Lit(D)$ such that for any literal $p \in At(D)$, $\{p, \neg p\} \in C$, and $\prec$ is an acyclic binary relation over the non-strict rules in $R$.
\end{definition}

The elements of $C$
 are called  \emph{conflict
sets}. It is these which determine incompatibilities in defeasible theories. Simply put, a collection of rules conflict if their heads constitute a conflict set.
The  priority relation $\prec$ is used to  resolve conflicts between non-strict rules, and it is this relation which in part determines which rules can be used to defeat others.
Different versions of defeasible logic
specify precisely how these components are used.

We call  a conflict set of the form $\{p, \neg p\}$ a \emph{minimal} conflict set, and we use $C_{MIN}$ to indicate that no conflict sets other than the minimal ones are defined in a given defeasible theory.   We say that conflict sets are \emph{closed under strict rules} if, for all $c \in C$, if $A \rightarrow p$ is a rule and $p \in c$, then $(A
\cup c-\{p\})\in C$.  It  is expensive  to close conflict sets under the strict rules of a theory, but it is often necessary to do so in order to draw  reasonable conclusions. The non-minimal conflict sets are called \emph{extended} conflict sets. The predecessors of NDL and ADL, including BDL and its variants, do not allow extended conflict sets.

For a given theory, $R_s$, $R_d$, and $R_u$ refer to the strict,
defeasible, and defeater rules of $R$, respectively, while
$R_s[p]$, $R_d[p]$, and $R_u[p]$  refer to those rules with head
$p$. $C[p]$ denotes the set of conflict sets containing $p$. The expression $\overline{p}$ refers to the complement of $p$.

\section{Well-founded semantics for ADL and NDL}\label{WFSSec}

The proof systems for ADL and NDL are presented in \ref{PS}. We will not discuss them further here other than to say that adequate semantics (presented below) corresponding to the logics did not exist until recently \cite{maier2009}. The  proof
systems for NDL and ADL are sound
relative to their counterpart semantics, and  while completeness does not hold in general, the proof systems are complete for the class of \emph{locally finite theories} (defined in \ref{PS}).  The restricted nature of the completeness result is unsurprising, since proofs in defeasible logic are required to be finite structures, while no counterpart restriction exists for the semantics.

As Example \ref{ambiex} illustrates,
 NDL and ADL differ in how ambiguity is handled.

\begin{example}\label{ambiex}\mbox{$D = \langle R, C_{MIN}, \myemptyset\rangle$, $R$ is}

\begin{enumerate}
\item  $\myemptyset \Rightarrow p$
\item  $\myemptyset \Rightarrow \neg p$
\item  $\{p\} \Rightarrow \neg q$
\item  $\myemptyset \Rightarrow q$
\end{enumerate}
\end{example}

The first two defeasible rules in the example are vacuously supported and conflict, and there is no mechanism for choosing between them. The literals $p$ and $\neg p$ are \emph{ambiguous} in an intuitive sense, and there is some debate in the literature regarding the proper handling of ambiguity.  It is clear that neither $p$ nor $\neg p$ should be considered justified.  One possible course of action is to consider both $p$ and $\neg p$ as refuted, which effectively \emph{blocks} or localizes the ambiguity to just those literals.  This is the course taken by Horty \cite{horty87} and by most forms of defeasible logic, including NDL.  If one does this, then since $p$ is refuted, all  support for $\neg q$
vanishes, and only $q$ is left with any support. Indeed, under the ambiguity blocking view, $q$ is proved while $\neg q$ is refuted.

Alternatively, one could simply refrain from concluding anything at all about $p$ and $\neg p$. Since the status of $\neg q$ depends upon resolving the status of $p$, the ambiguity of $p$ is effectively
 \emph{propagated} to $\neg q$.  This is the course taken by ADL.  Adopting ambiguity propagation yields in a sense a more extreme
form of skepticism, in that  \emph{fewer} conclusions can be drawn.
In the example, $p$ \emph{might} hold (there is conflicting
information about it and no way to resolve the conflict), and if it
does hold, then there would be evidence for both $q$ and $\neg q$, and so $q$ and $\neg q$ would be ambiguous.

The WFS for logic programs is also ambiguity propagating.  The following logic program $\Pi$ is the most-natural counterpart to the above defeasible theory.

\begin{enumerate}
\item  $p \neck \naf \neg p$
\item  $\neg p \neck \naf p$
\item  $\neg q \neck  \naf q, p$
\item  $q \neck \naf \neg q$
\end{enumerate}

\noindent
If $\neg p$ and $\neg q$ are simply treated as atoms (which is what the original WFS would do),   it can be seen that no non-empty subset of $At(\Pi)$ is an unfounded set relative to $\bot = \langle\myemptyset, \myemptyset\rangle$.  Furthermore, $T_\Pi(\bot)$ is empty. As this is so, the well-founded model of the program is simply $\bot$.

We note that in Example \ref{ambiex}, the ambiguity of $p$ and $\neg p$ can be resolved  by specifying that either rule 1 or 2 takes priority over the other. For instance, if $2 \prec 1$, then both NDL and ADL would conclude $p$, and $\neg p$ would be refuted. In both logics, $q$ and $\neg q$ would still be ambiguous, because neither rule 3 nor 4 is superior to the other.

The semantics for both NDL and ADL  are based explicitly on the WFS. As noted earlier, the components $\itrue{}$ and $\ifalse{}$  of interpretations are allowed to contain negative literals.    The set of interpretations still forms a complete lattice under $\sqsubseteq$, with $\top$ now being  $\langle Lit(D),Lit(D)\rangle$.
The operators $U_D$, $T_D$, and $W_D$, as well as the underlying notion of unfounded set, are recast to apply to defeasible theories.  Somewhat surprisingly, the difference between the semantics for NDL and that for ADL lies solely in how unfounded sets are defined.  It is this definition that determines whether ambiguity is blocked or propagated.

\begin{definition}\label{unfoundedSetDL}
A set $S \subseteq Lit(D)$ is \emph{unfounded} in NDL with respect
to $D$ and an interpretation $\tva{} = \langle \itrue{}, \ifalse{}\rangle{}$ iff for all literals $p \in
S$:

\begin{enumerate}

\item For every $r\in R_s[p]$, $body(r) \cap  (\ifalse{} \cup S) \neq
\myemptyset$.

\item For every $r\in R_d[p]$,

\begin{enumerate}

\item $body(r) \cap  (\ifalse{} \cup S) \neq \myemptyset$, or

\item there is a $c \in C[p]$ such that for each $q \in c-\{p\}$
there is a rule $s \in R[q]$ such that
\begin{enumerate}
\item  $body(s) \subseteq \itrue{}$ and,
\item\label{2bi}   $s \nprec r$.
\end{enumerate}
\end{enumerate}
\end{enumerate}
\end{definition}

\noindent{}
The
definition of unfounded set in ADL is exactly the same as for NDL,
save that condition \ref{2bi} is replaced with the following
requirement: \emph{$r \prec s$ or $s$ is strict}.

Examining the definition above and Example \ref{ambiex}, it can be seen that $\{p,\neg p\}$ is an unfounded set in NDL relative to $D$ and interpretation $\langle \varnothing,\varnothing\rangle$. It is not unfounded according to ADL, however.

Given an account of unfounded set, $U_D$ is defined identically for  NDL and ADL.

\begin{center}$U_D(\tva{}) = \bigcup\{S|$ $S$ is
an unfounded set \emph{wrt} $D$ and $\tva{}$\}\end{center}

The
immediate consequence operator for NDL and ADL is  defined in terms of \emph{witnesses of provability}.

\begin{definition} If  $D$ is a defeasible theory and $\tva{} = \langle \itrue{}, \ifalse{}\rangle$ an interpretation, then  a rule $r \in R_D$ is a \emph{witness of provability} for $p$ \emph{wrt} $D$ and $\tva{}$ if one of the below conditions applies.
\begin{enumerate}

\item $r \in R_s[p]$ and $body(r) \subseteq
\itrue{}$.

\item  $r \in R_d[p]$ and $body(r)
\subseteq \itrue{}$, and for each conflict set $c \in C[p]$, there
exists a $q \in c-\{p\}$ such that for all $s\in R[q]$, $s \prec
r$ or $body(s) \cap \ifalse{} \neq \myemptyset$.
\end{enumerate}
\end{definition}

\noindent{}
Given this, the \emph{immediate consequences} of $D$ \emph{wrt} $\tva{}$, written $T_D(\tva{})$, is the set

\begin{center}$T_D(\tva{}) = \{p|$ there exists a witness of provability for $p$ \emph{wrt} $D$ and $\tva{}\}$.\end{center}

\noindent{}The account of the operator $W$ remains unchanged from the WFS:

\begin{center}
$W_D(\tva{}) = \langle T_D(\tva{}), U_D(\tva{})\rangle $
\end{center}

\noindent{}Furthermore, $W_D$ may be used to define the monotonically increasing  sequence $(\tva{0}, \tva{1}, \ldots)$. The sequence is
\emph{coherent}, in the sense that $\itrue{\alpha}\cap\ifalse{\alpha} = \myemptyset$ for any $\alpha \geq 0$. As in the WFS, the well--founded model of the defeasible theory is defined to be the least fixpoint $lfp(W_D)$ of $W_D$:

 \begin{center}
 $wfm(D) =_{def} lfp(W_D)$
 \end{center}

\noindent{}Again, it is the case that this fixpoint coincides with the limit of the above sequence.

The well-founded models of  defeasible theories (under NDL or ADL) can be viewed as defining both a consequence relation $\dsmodels$ and an ``anti-consequence'' relation $\dsrefutes$. Analogous relations ($\dsmodels_{WFS}$ and $\dsrefutes_{WFS}$) can be defined for normal logic programs under the WFS.

\begin{definition}
Let $D$ be a defeasible theory, $L$ one of NDL or ADL, and $wfm(D) = \langle \itrue{}, \ifalse{}\rangle$ $D$'s well-founded model according to $L$.
\begin{enumerate}
\item $D \dsmodels_{L} p$ \emph{iff} $p \in \itrue{}$, and
\item $D
\dsrefutes_{L} p$ \emph{iff} $p \in \ifalse{}$.
\end{enumerate}
\end{definition}

\section{Translating Defeasible Theories into Logic Programs}\label{DL2LP}

In the scheme used by Brewka \citeyear{brewka} to translate defeasible theories into logic programs,
every defeasible rule $S \Rightarrow p$ becomes
$p \neck \naf \overline{p}, S$, and every strict rule $S \rightarrow p$ becomes $p \neck S$. The result of
the transformation is a so-called \emph{extended} logic program (which allows both $\naf$ and $\neg$ to be used).
Several examples are presented to demonstrate that the two systems do
not always agree, and Brewka argues that the results of the defeasible logic are less reasonable.  The particular logic investigated by Brewka is an ambiguity propagating variant of BDL (without team defeat, a feature we have not discussed).
Brewka dismisses other variants of defeasible logic
without discussing them in any detail, mainly because these logics are ambiguity blocking.

Brewka's translation scheme assumes minimal conflict sets.
 Below, we alter it to encompass theories with extended conflict sets, and we use this modified scheme to compare ADL to the simple WFS, i.e., not to Brewka's prioritized variant. Since conflict sets are sufficient to encode
negation, we will assume that all negative literals are just atoms. Furthermore, since defeaters and  priorities on rules are not defined
for the simple WFS, we assume that no defeaters occur in the theory and that $\prec$ is empty.

\begin{definition}Let $D = \langle R, C, \prec\rangle$ be a defeasible theory. For any literal $p \in Lit(D)$,

\begin{center}$Prod(C[p]) = \{\{a_1,\ldots,a_m\}|(a_1,\ldots,a_m) \in c_1-\{p\} \times \ldots \times c_m-\{p\}\}$\end{center}
\noindent{} where $C[p] = \{c_1,\ldots, c_m\}$
\end{definition}

$Prod(C[p])$ is the set of all sets that can be created by taking a single literal (other than $p$) from
each conflict set containing $p$ (the order in the n-ary product above does not matter). We use these sets when translating defeasible rules of a theory into
logic program rules. In order to ensure that the rules in the translation are finite in length, we require that $C[p]$ is finite for each $p \in  Lit(D)$.

\begin{definition}Let $D = \langle R, C, \myemptyset\rangle$ be a defeasible theory such that  $R_u = \myemptyset$ and $C[p]$ is finite for each $p \in Lit(D)$. The logic program translation $\Pi_D$ of $D$ is the smallest rule-set such that

\begin{enumerate}
\item If $\{q_1,\ldots,q_n\} \rightarrow p \in R_s$, then $p \neck q_1,\ldots,q_n \in \Pi_D$.
\item If $\{q_1,\ldots,q_n\} \Rightarrow p \in R_d$ and $\{a_1,\ldots,a_m\} \in Prod(C[p])$, then \\$p \neck \naf a_1,\ldots,\naf a_m,q_1,\ldots,q_n \in \Pi_D$.

\end{enumerate}
\end{definition}

Let $trans(r)$ denote the set of logic program rules obtained from rule $r$ of the defeasible theory.
If we ignore notational differences, then $trans(r) = \{r\}$ if $r$ is strict.  For the sake of convenience, we will simply say that $trans(r) = r$. Normally, if $r$ is defeasible, then $trans(r)$ will contain many rules, but  if conflict sets are minimal, then it again holds that $trans(r)$ contains only a single rule. In the translation of a defeasible rule $r$, each $a_i$ is a literal of some conflict set containing $p$.

\subsection{Soundness and Completeness of ADL \emph{wrt} WFS}\label{sec:SoundnessAndCompleteness}

Provided $C[p]$ is finite for each $p \in Lit(D)$, ADL is sound \emph{wrt} the WFS. If in addition conflict sets are minimal, ADL is  complete \emph{wrt} the WFS.

Below, since the operators for the WFS have direct analogs for defeasible logic and are defined individually for
each logic program and defeasible theory, we can use the same basic symbols for each (writing, for
instance, $T_D$ and $T_\Pi$) without causing confusion. We will use $(\tva{D,0}, \tva{D,1},\ldots)$ to denote the sequence
of interpretations obtained using $W_D$, and $(\tva{\Pi,0}, \tva{\Pi,1},\ldots)$ to denote the sequence of
interpretations obtained using $W_\Pi$. For a given interpretation $\tva{\Pi,\lambda}$, we will write $\itrue{\Pi,\lambda}$ and $\ifalse{\Pi,\lambda}$ to
distinguish well-founded and unfounded sets.
We will also sometimes indicate the well-founded model of a given theory $D$ (or program $\Pi$) by writing  $\tva{D,WF}$ ($\tva{\Pi,WF}$).

\begin{proposition}[Soundness of ADL \emph{wrt} WFS]\label{soundness}
Let $D = \langle R,C,\myemptyset\rangle$ be a defeasible theory such
that $R_u = \myemptyset$ and for each $p \in Lit(D)$, $C[p]$ is finite. Let $\Pi$ be the logic program translation of $D$.
For any $p \in Lit(D)$,

\begin{enumerate}
\item if $D \asmodels p$, then $\Pi \dsmodels_{WFS} p$, and
\item if $D \asrefutes p$, then $\Pi \dsrefutes_{WFS} p$.
\end{enumerate}
\end{proposition}

\begin{proof*}
The proof is by induction on the sequence $(\tva{D,0},\tva{D,1},\ldots)$, showing that for all $\lambda \geq 0$, if $p \in \itrue{D,\lambda}$ ($p \in \ifalse{D,\lambda}$), then $p \in \itrue{\Pi,WF}$ ($p \in \ifalse{\Pi,WF}$).  Since $\tva{D,0} = \tva{\Pi,0}$, the claim holds for $\lambda = 0$.  Suppose it holds for all $\kappa < \lambda$.
We may assume \emph{wlog}
that $\lambda$ is a successor ordinal.
There are two cases to consider.

\begin{enumerate}
\item Suppose $p \in \itrue{D,\lambda}$.  Then there exists an $r \in R[p]$
such that $body(r) \subseteq  \itrue{D,\lambda-1}$ and either (1) r is strict or else (2) r is defeasible and for each conflict
set $c \in C[p]$, there exists a $q \in c-\{p\}$ such that for all $s \in R[q]$, $body(s) \cap \ifalse{D,\lambda-1} \neq \myemptyset$.
In both cases, $body(r) \subseteq  \itrue{\Pi,WF}$ by the inductive hypothesis. As such, for any $r' \in trans(r)$,
$body(r')^{+} \subseteq \itrue{\Pi,WF}$.
If $r$ is strict then $body(r')^{+} = body(r')$ and so
 by definition of $T_\Pi$, $p \in \itrue{\Pi,WF}$.

 Suppose, instead, that $r$ is defeasible, and let $c \in C[p]$.  Then there exists a $q \in c-\{p\}$ such that for each $s \in R[q]$, $body(s) \cap \ifalse{D,\lambda-1} \neq \myemptyset$. Let $s' \in trans(s)$.  Since $body(s) = body(s')^{+}$, it follows that $body(s')^{+} \cap \ifalse{D,\lambda-1} \neq \myemptyset$. By inductive hypothesis,  $body(s')^{+} \cap \ifalse{\Pi,WF} \neq \myemptyset$.
Generalizing on $s'$ and then $s$, every rule for $q$ in $\Pi$ has a non-default literal in $\ifalse{\Pi,WF}$, and so by definition of $U_\Pi$, $q \in \ifalse{\Pi,WF}$.

Generalizing on $c$, every conflict set for $p$ has a literal $q \neq p$ such that $q \in \ifalse{\Pi,WF}$.  Let $Q=\{q_1,\ldots,q_m\}$ be the set of such literals. Obviously, $Q \in Prod(C[p])$, and so there is a rule $r' \in trans(r)$ such that $body(r')^{+} = body(r)$ and $body(r')^{-} = Q$.  Since $body(r')^{+} \subseteq \itrue{\Pi,WF}$, and $Q \subseteq \ifalse{\Pi,WF}$, it follows that $p \in \itrue{\Pi,WF}$.

\item Suppose $p \in \ifalse{D,\lambda}$ and let $b$ be any literal in $\ifalse{D,\lambda}$.  $\ifalse{D,\lambda}$ is by definition
unfounded wrt $D$ and $\tva{D,\lambda-1}$. If $r \in R_s[b]$, then there is a $q \in body(r)$ such that $q \in
\ifalse{D,\lambda}\cup \ifalse{D,\lambda-1}$. $U_D$ is monotone, and so $q \in \ifalse{D,\lambda}$.

Suppose $r \in R_d[b]$. Then either (1) there is a $q \in body(r)$ such that $q \in \ifalse{D,\lambda}$, or (2) there is
a conflict set $c \in C[b]$ such that for all $a \in  c-\{b\}$, there is a $s \in R_s[a]$ such that $body(s) \subseteq
\itrue{D,\lambda-1}$ ($s$ must be strict since the priority relation is empty). Suppose (2) holds. By inductive
hypothesis, $body(s) \subseteq  \itrue{\Pi,WF}$. Since $s$ is strict, $trans(s) = s$ and so $a \in \itrue{\Pi,WF}$. Recall that by definition of $Prod(C[b])$, for each set $Q \in Prod(C[b])$ we have
$Q\cap c-\{b\} \neq \myemptyset$. By definition of $trans(r)$, for each $t \in trans(r)$, there exists a $Q \in Prod(C[b])$
such that $Q = body(t)^{-}$. Since this is so, if (2) holds then for each rule $r' \in trans(r)$, there exists
a $a \in body(r')^{-}$ such that $a \in \itrue{\Pi,WF}$.

Generalizing on $r$, every logic program rule $r'$ for $b$ has a classical literal $q \in body(r')^{+}$ such
that $q \in \ifalse{D,\lambda}$, or else a default literal $a \in body(r')^{-}$ such that $a \in \itrue{\Pi,WF}$. Generalizing on $b$, by definition
of unfounded sets for logic programs, $\ifalse{D,\lambda}$ is an unfounded set relative to $\Pi$ and $\tva{\Pi,WF}$, and so
$\ifalse{D,\lambda}\subseteq \ifalse{\Pi,WF}$. Since $p \in \ifalse{D,\lambda}$, $p \in \ifalse{\Pi,WF}$.
\mathproofbox\end{enumerate}
\end{proof*}

As noted, the claim of soundness pertains to defeasible theories with extended conflict sets,
provided that $C[p]$ is finite for all $p$. The other specific requirements are that no defeaters are used
and that the rules are unprioritized. Completeness requires more, however. Specifically,  conflict sets must also be minimal.

\begin{example}\label{ADLdiffersWFS} Consider the following unprioritized defeasible theory and its corresponding logic program, and
suppose that the conflicts sets $\{p,\neg p\}$, $\{q,\neg q\}$, $\{q, \neg p\}$ are used.

\begin{center}
\begin{tabular}{cc}
\begin{minipage}{.4\textwidth}
\begin{enumerate}
\item $\myemptyset \rightarrow p$
\item $\myemptyset \Rightarrow \neg p$
\item $\myemptyset \Rightarrow q$
\item $\{q\} \rightarrow p$
\end{enumerate}
\end{minipage}
&
\begin{minipage}{.4\textwidth}
\begin{enumerate}
\item $p$
\item $\neg p \neck \naf p$
\item $q \neck \naf \neg q, \naf \neg p$
\item $p \neck q$
\end{enumerate}
\end{minipage}\\
\end{tabular}
\end{center}
 In the defeasible theory, $p$ is a fact, and so the presumption of $\neg p$ is defeated. Nevertheless, it is supported, and this is sufficient to prevent $q$ from being concluded---$q$ is ambiguous according to ADL. In the logic program, however, both $p$ and $q$ are well-founded.  And so the two formalisms disagree. In contrast, if the conflict sets are minimal, the defeasible theory and corresponding logic program yield the same results.
 Examples such as above are problematic for ADL and NDL, as the result produced by the logic program is intuitively more reasonable than the one produced by the defeasible theory with conflict sets closed under strict rules. We pick up this topic again in Section \ref{relatedwork}.
\end{example}

Before we prove that ADL is complete with respect to the WFS---subject to the restrictions noted above---we need to prove the following small lemma.

\begin{lemma}\label{completenessLemma} Let $D = \langle R,C,\myemptyset\rangle$  be a defeasible theory such that $R_u = \myemptyset$ and the conflict sets of $C$ are
minimal. For all $p \in Lit(D)$, if $r \in R_d[p]$, $body(r) \subseteq \itrue{WF}$, and $\overline{p} \in  \ifalse{WF}$, then $p \in \itrue{WF}$.\end{lemma}

\begin{proof*}
Suppose $r \in R_d[p]$ and $body(r) \subseteq \itrue{WF}$ and $\overline{p} \in \ifalse{WF}$. Then there must be some least
successor ordinal $\lambda$ such that $body(r) \subseteq \itrue{\lambda}$ and $\overline{p}\in \ifalse{\lambda}$. Recall that $\ifalse{\lambda}$ is the greatest unfounded
set wrt $\tva{\lambda-1}$. Suppose for a proof by contradiction that $p \notin \itrue{\lambda+1}$. By definition of unfounded set (and since $C = C_{MIN}$) we have,

\begin{enumerate}
\item for all $s \in  R_s[\overline{p}]$, $body(s) \cap (\ifalse{\lambda} \cup \ifalse{\lambda -1})\neq \myemptyset$, and
\item  for all $s \in R_d[\overline{p}]$, either
\begin{enumerate}
\item $body(s) \cap  (\ifalse{\lambda} \cup \ifalse{\lambda-1})\neq \myemptyset$, or
\item\label{2b} there is a rule $t \in R_s[p]$ such that $body(t) \subseteq \itrue{\lambda-1}$.
\end{enumerate}
\end{enumerate}
Since $U_D$ is monotone,
$body(s) \cap (\ifalse{\lambda} \cup \ifalse{\lambda -1})$
reduces to $body(s) \cap \ifalse{\lambda}$. If \ref{2b} above holds, then $p \in \itrue{\lambda}$ and so $p \in \itrue{\lambda+1}$.  As such, $r \in R_d[p]$ and $body(r) \subseteq \itrue{\lambda}$, and it must be that for each rule $s \in R[\overline{p}]$,
$body(s) \cap \ifalse{\lambda}\neq \myemptyset$. But  this implies (via $T_D$) that $p \in \itrue{\lambda+1}$.  This is a contradiction, and so (again) $p \in \itrue{\lambda+1}$.\mathproofbox\end{proof*}

\begin{proposition}[Completeness]\label{completeness}
Let $D = \langle R,C,\myemptyset\rangle$ be a defeasible theory such that $R_u = \myemptyset$ and the conflict sets of $C$ are minimal.  Let $\Pi$ be the logic program translation of $D$. For all $p \in Lit(D)$,
\begin{enumerate}
\item if $\Pi \dsmodels_{WFS} p$, then $D \asmodels p$, and
\item if $\Pi \dsrefutes_{WFS} p$, then $D \asrefutes p$.
\end{enumerate}
\end{proposition}

\begin{proof*}

\noindent
The proof is by induction on the sequence $(\tva{\Pi,0}$, $\tva{\Pi,1}$,$\ldots)$, showing that for all $\lambda \geq 0$, if $p \in \itrue{\Pi,\lambda}$ ($p \in \ifalse{\Pi,\lambda}$), then $p \in \itrue{D,WF}$ ($p \in \ifalse{D,WF}$). Since $\tva{D,0} = \tva{\Pi,0}$, the clam holds for $\lambda = 0$.  Suppose it holds for all $\kappa < \lambda$.  We may assume that $\lambda$ is a successor ordinal.

\begin{enumerate}
\item Suppose $p \in \itrue{\Pi,\lambda}$.  Then there is a rule $s \in trans(r)$ for some $r \in R[p]$ such that $body(s)^{+}\subseteq \itrue{\Pi,\lambda-1}$ and $body(s)^{-} \subseteq \ifalse{\Pi,\lambda-1}$.  By inductive hypothesis, $body(s)^{+}\subseteq \itrue{D,WF}$ and $body(s)^{-} \subseteq \ifalse{D,WF}$.  If $body(s)^{-} = \myemptyset$, then $r \in R_s[p]$ and $body(r) = body(s)$, and so $p \in \itrue{D,WF}$ by definition of $\tva{D,WF}$ and $T_D$. If $body(s)^{-} \neq \myemptyset$, since conflict sets are minimal, it must be that  $body(s)^{-} = \{\overline{p}\}$.  And so $\overline{p} \in \ifalse{D,WF}$.  Since $body(s)^{+} = body(r)$, by Lemma \ref{completenessLemma}, $p \in \itrue{D,WF}$.

\item Suppose $p \in \ifalse{\Pi,\lambda}$.
Let $a$ be a literal of $\ifalse{\Pi, \lambda}$ and let $r' \in trans(r)$ for some $r \in R[a]$. If $r \in R_s[a]$, then since $U_\Pi$ is monotone and $\ifalse{\Pi,\lambda}$ is unfounded relative to $\tva{\Pi,\lambda-1}$, there exists a $b \in body(r')^{+}$ such that $b \subseteq \ifalse{\Pi,\lambda}$.  Thus there exists a $b \in body(r)$ such that $b \in  \ifalse{\Pi,\lambda}$.

Suppose that $r$ is defeasible. Then either (1) there exists a classical $b \in body(r')$ such that $b \in \ifalse{\Pi,\lambda}$, or else (2) the  literal $\naf \overline{a}$ appears in $body(r')$ and $\overline{a} \in \itrue{\Pi,\lambda-1}$. If (1), then
$body(r) \cap \ifalse{\Pi,\lambda} \neq \myemptyset$. If (2) then by the inductive hypothesis $\overline{a} \in \itrue{D,WF}$, and so there must be
a rule $s \in R[\overline{a}]$ such that $body(s) \subseteq \itrue{D,WF}$ and either (2.1) $s$ is strict or else (2.2) for all rules
$t \in R[a]$ (including $t = r$), $body(t) \cap \ifalse{D,WF} \neq \myemptyset$.

Generalizing on $r$, for each rule $r \in R_s[a]$, $body(r) \cap (\ifalse{\Pi,\lambda} \cup \ifalse{D;WF}) \neq \myemptyset$. For each $r \in
R_d[a]$, either
$body(r) \cap (\ifalse{\Pi,\lambda} \cup \ifalse{D;WF}) \neq \myemptyset$ or else there exists an $s \in R_s[\overline{a}]$ and $body(s) \subseteq
\itrue{D,WF}$. Generalizing on $a$,  $\ifalse{\Pi,\lambda}$ is   unfounded  \emph{wrt} $D$ and $\tva{D,WF}$, and so
 $\ifalse{\Pi,\lambda}\subseteq U_D(\tva{D,WF})$. Since $\tva{D,WF}$ is a fixpoint of $W_D$, we have $\ifalse{\Pi,\lambda}\subseteq \ifalse{D,WF}$ and
hence $p \in \ifalse{D,WF}$. \mathproofbox
\end{enumerate}
\end{proof*}

\section{Translating Logic Programs into Defeasible Theories}\label{LPtoDL}

A translation in the other direction is also possible.
That is, normal logic programs under the WFS can also be translated into defeasible theories under ADL so that the canonical models of each agree. We show this by first translating the normal logic program into
an equivalent extended logic program  encoding the closed-world assumption.  We then show the equivalence
between the ADL theory and the extended program.  Below, though  both $\naf$ and $\neg$ appear  in the extended logic program, we intend the original WFS to be used on the programs---a literal $\neg p$ is simply taken as another atom.

\begin{definition}Let $\Pi$ be a normal program. The \emph{explicit version} of $\Pi$
is the smallest extended program $\Phi$ such that
\begin{enumerate}
\item If $p \neck  a_1,\ldots,a_n, \naf  b_1,\ldots,\naf b_m$ appears in $\Pi$, then\\
 $p \neck  a_1,\ldots,a_n, \neg b_1,\ldots,\neg b_m$ appears in $\Phi$.

\item For each $p \in At(\Pi)$, the rule $\neg p \neck \naf p$ appears in $\Phi$.
\end{enumerate}
\end{definition}
The following lemmas relate the well-founded models of $\Pi$ and $\Phi$ and make the translation of $\Pi$ into a defeasible theory $D_\Pi$ apparent (Definition \ref{labelDefDL2LP}).

\begin{lemma}\label{explicitNF1}
Let $\Pi$ be a normal program and $\Phi$ its explicit version. For any $b \in At(\Pi)$ and any
ordinal $\lambda \geq 0$,
\begin{enumerate}
\item $\neg b \in \itrue{\Phi,\lambda}$ \emph{iff}
 there exists a $\kappa < \lambda$ such that $b \in \ifalse{\Phi,\kappa}$.
\item $\neg b \in \ifalse{\Phi,\lambda}$ \emph{iff}
 there exists a $\kappa < \lambda$ such that $b \in \itrue{\Phi,\kappa}$.
\end{enumerate}
\end{lemma}

\begin{proof*}
\begin{enumerate}
\item Suppose $\neg b \in \itrue{\Phi, \ordK}$ for some ordinal $\ordK$. Then there exists a least successor ordinal $\ordJ \leq \ordK$ such that $\neg b \in \itrue{\Phi, \ordJ}$.  As $\neg b \neck \naf b$ is the only rule with head $\neg
b$, it must be the case that  $b \in \ifalse{\Phi,\ordJ-1}$.

Now suppose  there is an ordinal $\ordJ < \ordK$  such that $b \in \ifalse{\Phi,\ordJ}$.
Since $\neg b \neck \naf b$ is a rule in $\Phi$, it must be the
case that $\neg b \in \itrue{\Phi,\ordJ+1}$. Either $\ordJ+1 = \ordK$, or else by monotonicity of the sequence $(\tva{})$ we have $\neg b \in \itrue{\Phi,\ordK}$.

\item Suppose $\neg b \in \ifalse{\Phi, \ordK}$. Then there exists a least successor ordinal $\ordJ \leq \ordK$ such that $\neg b \in \ifalse{\Phi, \ordJ}$.  Since $\neg b \neck \naf b$ is the only rule in $\Phi$ with head $\neg b$, it must be the
case that $b \in \itrue{\Phi,\ordJ-1}$.

Now suppose  there is an ordinal $\ordJ < \ordK$ such that $b \in \itrue{\Phi,\ordJ}$.  As $\neg b \neck \naf b$ is the only  rule of $\Phi$ with head $b$, it must be the case that  $\neg b \in \ifalse{\Phi,\ordJ+1}$.
  Either $\ordJ+1 = \ordK$, or else by monotonicity we have $\neg b \in \ifalse{\Phi,\ordK}$. \mathproofbox
\end{enumerate}
\end{proof*}

\noindent{}It immediately follows from the above Lemma that $\Phi \dsmodels_{WFS} b$ \emph{iff} $\Phi \dsrefutes_{WFS} \neg b$, and $\Phi \dsmodels_{WFS} \neg b$ \emph{iff} $\Phi \dsrefutes_{WFS} b$.

\begin{lemma}\label{explicitNF2}
Let $\Pi$ be a normal program and $\Phi$ the explicit version of $\Pi$. For any $b \in At(\Pi)$

\begin{enumerate}
\item $\Pi \dsmodels_{WFS} b$ \emph{iff} $\Phi \dsmodels_{WFS} b$.
\item $\Pi \dsrefutes_{WFS} b$ \emph{iff} $\Phi \dsrefutes_{WFS} b$.
\end{enumerate}
\end{lemma}
\begin{proof*}
\noindent{}\textbf{(LR)}  The proof is by induction on the sequence $(\tva{\Pi, 0}, \tva{\Pi, 1}$, $\ldots)$.  Suppose for all $\kappa{} < \ordK$ and all $p \in At(\Pi)$, if $p \in \itrue{\Pi, \kappa{}}$, then $p \in \itrue{\Phi,WF}$; if  $p \in \ifalse{\Pi, \kappa{}}$ then $p \in \ifalse{\Phi, WF}$.  We may assume \emph{wlog} that $\ordK$ is a successor ordinal.

\begin{enumerate}
\item
 Suppose  $p \in \itrue{\Pi, \ordK}$. Then there is a rule $r$ with head $p$
such that $body(r)^{+} \subseteq \itrue{\Pi, \ordK-1}$ and $body(r)^{-}
\subseteq \ifalse{\Pi, \ordK-1}$. Let $r'$ be the rule of $\Phi$
corresponding to $r$. By the inductive hypothesis, $body(r)^{+} \subseteq
\itrue{\Phi,WF}$ and $body(r)^{-} \subseteq \ifalse{\Phi,WF}$.  By  Lemma \ref{explicitNF1} for each $q \in body(r)^{-}$  it follows that $\neg q \in  \itrue{\Phi,WF}$.  Since $body(r)^{+} \subseteq \itrue{\Phi, WF}$ and for each $q \in body(r)^{-}$ we have $\neg q \in  \itrue{\Phi,WF}$, it must be the case that $body(r') \subseteq \itrue{\Phi, WF}$.  Since $r'$ is strict, then by definition of $T_{\Phi}$ and $\tva{\Phi,WF}$,  $p \in \itrue{\Phi,WF}$.

\item Suppose $p \in \ifalse{\Pi,\ordK}$.  Let $q \in At(\Pi)$ be any literal such that $q \in \ifalse{\Pi,\ordK}$.  Then for
all $r\in R_\Pi[q]$ there is an $a \in
body(r)^{+}$ such that $a \in \ifalse{\Pi,\ordK}$ or else a
$\naf b\in body(r)^{-}$ such that $b \in \itrue{\Pi,\ordK-1}$.  If $b \in \itrue{\Pi,\ordK-1}$, then by the inductive hypothesis $b \in \itrue{\Phi, WF}$ and so from Lemma \ref{explicitNF1}  $\neg b \in \ifalse{\Phi,WF}$.  Generalizing, for all rules $r$ for $q$, each
corresponding rule $r'$ has a classical literal $a \in body(r')$
such that $a \in \ifalse{\Pi,\ordK}$ or else a (still classical literal)
$\neg b$ such that  $\neg b \in \ifalse{\Phi,WF}$. Generalizing on
$q$,   $\ifalse{\Pi,\ordK}$ is  unfounded  \emph{wrt}
 $\Phi$ and $\tva{\Phi,WF}$.  As such,   $\ifalse{\Pi,\ordK} \subseteq U_\Phi(\tva{\Phi,WF}) = \ifalse{\Phi,WF}$, and so $p \in \ifalse{\Phi,WF}$.

\end{enumerate}

\noindent{}\textbf{(RL)}  The proof is by induction on the sequence $(\tva{\Phi, 0}, \tva{\Phi, 1}$, $\ldots)$.  Suppose for all $\ordJ < \ordK$ and $p \in At(\Pi)$, if
 $p \in \itrue{\Phi, \ordJ}$, then  $p \in \itrue{\Pi, WF}$; if  $p \in \ifalse{\Phi, \ordJ}$ then  $p \in \ifalse{\Pi, WF}$. We may assume \emph{wlog} that $\ordK$ is a successor ordinal. Let $p$ be any atom of $At(\Pi)$.

\begin{enumerate}
\item Suppose $p \in \itrue{\Phi, \ordK}$.  Then there is a rule $r'$ with
head $p$ such that $body(r') \subseteq \itrue{\Phi,\ordK-1}$.  By Lemma \ref{explicitNF1}, for each $\neg b \in
body(r')$, we have $b  \in \ifalse{\Phi,\ordG}$ for some $\ordG < \ordK$. Let $r$ be the rule of $\Pi$
corresponding to $r'$. By the inductive hypothesis, $body(r)^{+} \subseteq
\itrue{\Pi,WF}$ and $body(r)^{-} \subseteq \ifalse{\Pi,WF}$. By definition of $T_{\Pi}$ and $\tva{\Pi,WF}$,  $p \in \itrue{\Pi,WF}$.

\item Now suppose $p \in \ifalse{\Phi,\ordK}$ and let $q$ be any literal such that $q \in \ifalse{\Phi,\ordK}$.  If $q$ is a classical negative literal, then no rules for $q$ appear in $\Pi$.  Suppose $q$ is an atom. Since $q \in
\ifalse{\Phi,\ordK}$, for all rules $r'$ with head $q$ there is a
classical literal $a \in body(r)$ such that $a \in \ifalse{\Phi,\ordK}$.
If $a$ is of the form $\neg b$, then by Lemma \ref{explicitNF1}, $b \in
\itrue{\Phi,\ordG}$ for some $\ordG < \ordK$. By the inductive hypothesis, each such $b$ is in
$\itrue{\Pi,WF}$. Recall that if $\neg b$ appears in the body of $r'$, then $\naf b$ appears in the corresponding rule $r$ of $\Pi$.  Generalizing on $r'$, for each rule $r$ in $\Pi$ with head $q$, there is an $a\in body(r)^{+}$ such that $a \in
\ifalse{\Phi,\ordK}$, or else a $b \in body(r)^{-}$ such that $b
\in \itrue{\Pi,WF}$. Generalizing on $q$, $\ifalse{\Phi,\ordK}$ is
unfounded  \emph{wrt} $\Pi$ and $\tva{\Pi,WF}$.  As such, $\ifalse{\Phi,\ordK} \subseteq U_\Pi(\ifalse{\Pi, WF}) = \ifalse{\Pi, WF}$, and so $p \in \ifalse{\Pi,WF}$.\mathproofbox
\end{enumerate}
\end{proof*}

\begin{definition}\label{labelDefDL2LP} Let $\Pi$ be a normal logic program. If rule $r$

\begin{center}
$p \neck  a_1,\ldots,a_n, \naf b_1,\ldots,\naf b_m$
\end{center}

 appears in $\Pi$, then $r_{D_\Pi}$ is the rule

\begin{center}
$\{a_1,\ldots,a_n, \neg b_1,\ldots,\neg b_m\} \rightarrow p$
\end{center}
\end{definition}

\begin{definition}
If  $\Pi$ is a normal logic program, then the defeasible theory translation $D_\Pi$ of $\Pi$ is $\langle Str\cup Pr, C_{MIN}, \myemptyset\rangle$, where

\begin{enumerate}
\item $Str = \{r_{D_\Pi}|r \in \Pi\}$.
\item $Pr = \{\myemptyset \Rightarrow \neg p|p \in At(\Pi)\}$.
\end{enumerate}

\end{definition}

The default literals in the program have become presumptions in the defeasible theory. The
rules of the original program are strict in the defeasible theory. It should be obvious that translating $D_\Pi$ back into a logic
program using the Brewka inspired scheme yields $\Phi$. Given the soundness and completeness
 results of the last section and also Lemma \ref{explicitNF1}, it follows that $p$ is well-founded in $D_\Pi$ under ADL
if and only if $\neg p$ is unfounded under ADL, and $\neg p$ is well-founded in ADL if and only if $p$ is unfounded in
ADL. Given Lemma \ref{explicitNF2}, the
results of $D_\Pi$ under ADL agree with those of $\Pi$ \emph{wrt} $At(\Pi)$.

\begin{proposition}
If $\Pi$ is a normal logic program, then for
any $p \in At(\Pi)$,

\begin{enumerate}
\item $D_\Pi \asmodels p$ iff $D_\Pi \asrefutes \neg p$,
\item $D_\Pi \asmodels \neg p$ iff $D_\Pi \asrefutes p$
\end{enumerate}

\end{proposition}
\begin{proof*}Let $\Phi$ be the explicit normal form of $\Pi$, and
suppose $D_\Pi \asmodels p$.  By Prop. \ref{soundness}, $\Phi \dsmodels_{WFS} p$.
By Lemma \ref{explicitNF1}, $\Phi \dsrefutes_{WFS} \neg p$.
By Prop. \ref{completeness}, $D_\Pi \asrefutes \neg p$.
Now suppose   $D_\Pi \asrefutes p$.  By Prop. \ref{soundness}, $\Phi \dsrefutes_{WFS} p$.
By Lemma. \ref{explicitNF1}, $\Phi \dsmodels_{WFS} \neg p$.
By Prop. \ref{completeness}, $D_\Pi \asmodels \neg p$.
The remaining cases are analogous. \mathproofbox
\end{proof*}

\begin{proposition}
If $\Pi$ is a normal logic program, then for
any $p \in At(\Pi)$,

\begin{enumerate}
\item $\Pi \dsmodels_{WFS} p$ \emph{iff} $D_\Pi \asmodels p$.
\item $\Pi \dsrefutes_{WFS} p$ \emph{iff} $D_\Pi \asrefutes p$.
\end{enumerate}

\end{proposition}
\begin{proof*}
Suppose $D_\Pi \asmodels p$. Then $\Phi \dsmodels_{WFS} p$ by Prop. \ref{soundness}. By Lemma \ref{explicitNF2}, $\Pi \dsmodels_{WFS} p$. Now
suppose $\Pi \dsmodels_{WFS} p$. Then $\Phi \dsmodels_{WFS} p$ by Lemma \ref{explicitNF2}. By Prop. \ref{completeness}, $D_\Pi \asmodels p$.
Again, the remaining cases are analogous.  \mathproofbox
\end{proof*}

\begin{example} A logic program $\Pi$, its explicit form $\Phi$, and its defeasible logic translation $D_\Pi$ are
shown below.

\begin{center}
\begin{tabular}{ccc}\hline
$\Pi$ & $\Phi$ & $D_\Pi$\\\hline
\begin{minipage}{.30\textwidth}
\begin{enumerate}
\item $p \neck \naf q$
\item $q \neck \naf p$
\item[]
\item[]
\end{enumerate}
\end{minipage}
 &
\begin{minipage}{.30\textwidth}
\begin{enumerate}
\item $p \neck \neg q$
\item $q \neck \neg p$
\item $\neg q \neck \naf q$
\item $\neg p \neck \naf p$
\end{enumerate}
\end{minipage}

&
\begin{minipage}{.30\textwidth}
\begin{enumerate}
\item $\{\neg q\} \rightarrow p$
\item $\{\neg p\} \rightarrow q$
\item $\myemptyset \Rightarrow \neg q$
\item $\myemptyset \Rightarrow \neg p$
\end{enumerate}
\end{minipage}\\\hline
\end{tabular}
\end{center}
\end{example}

In the rules of $\Pi$, we have replaced each $\naf a$  (where $a$ is an atom) with $\neg a$ and added the rules $\neg a \neck \naf a$. The explicitly negative literals
occur nowhere in $\Pi$. The well-founded model of both $\Pi$ and $\Phi$ is empty. In $D_\Pi$, in order to show $\neg p$, we must first show that $\neg q$ is unfounded, and furthermore, to show
$\neg q$, we must first show that $\neg p$ is unfounded. Because of this nothing can be determined in ADL about $p$, $\neg p$, $q$, or $\neg q$.

\section{Antimonotone Operators for ADL and NDL}\label{AltFixpoints}

The  account of the WFS provided above is found in \cite{vangelder91}.  It is, however, more typical today to present the WFS in terms of the so-called Gelfond-Lifschitz (GL) operator $\GLop$, which was first defined for the stable model semantics \cite{gelfond88}.  In this section, we review the definition of $\GLop$ and use it to define ambiguity blocking and propagating operators for defeasible theories. With some restrictions, these can be used to calculate the consequences of theories according to ADL and NDL. As shown in the next section, they can also be used to define stable model semantics for defeasible theories.

The GL operator $\GLop$  works with Herbrand interpretations---sets of ground atoms. If $S$ is a Herbrand interpretation,  then atom $p$ is true in $S$ if $p \in S$, and false if $p \notin S$.
If $\Pi$ is a normal logic program and $S$ a Herbrand interpretation,
then

\begin{center}$\GLop_\Pi(S) =_{def} T_{\Pi^S}\uparrow \omega$
 \end{center}

 \noindent{}where $\Pi^S$ is the so-called \emph{reduct} of $\Pi$ \emph{wrt} $S$. Specifically, $\Pi^{S}$ is the NAF-free program  obtained by

\begin{enumerate}
\item deleting from $\Pi$ all rules $r$ such that $body(r)^{-} \cap
S \neq \myemptyset$.
\item deleting all remaining default literals.
\end{enumerate}

For an NAF-free program $\Pi$, the immediate consequence operator $T_\Pi$ reduces to

\begin{center}$T_\Pi(S) =_{def} \{head(r) |$ $r \in \Pi$ and $body(r) \subseteq
S\}$\end{center}

The sequence $T_\Pi\uparrow{}0$, $T_\Pi\uparrow{}1$, $\ldots$, is  defined for ordinals $\lambda \geq 0$.

\begin{enumerate}
\item $T_\Pi\uparrow{}0 = \myemptyset$
\item$T_\Pi\uparrow{}\lambda+1 = T_\Pi(T_\Pi\uparrow{}\lambda)$ (for successor ordinals $\lambda+1$)
\item$T_\Pi\uparrow{}\lambda = \bigcup_{\kappa < \lambda} T_\Pi \uparrow \kappa$ (for limit ordinals $\lambda$)
\end{enumerate}

The set of Herbrand interpretations forms a complete lattice under $\subseteq$, and it is also the case that $T_\Pi$ is continuous on this lattice.  As such, $lfp(T_\Pi) = T_P\uparrow \omega$ \cite{vanemden}.
$T_\Pi\uparrow \omega$ is sometimes written as $Cl(\Pi)$, and so if $\Pi$ is normal, then $\gamma_\Pi(S) = Cl(\Pi^S)$.
%$S$ is called a \emph{stable model} of $\Pi$ if $\gamma_\Pi(S) = S$.

The $\gamma$ operator is antimonotone, and so $\gamma^2$ is monotone.  As shown by Baral and Subrahmanian \citeyear{baral91}, the well-founded model of $\Pi$ can be defined in terms of $\gamma^2_\Pi$.  Specifically,
 $wfm(\Pi) = \langle\itrue{},\ifalse{}\rangle$, where $\itrue{} = lfp(\gamma_{\Pi}^2)$, and $\ifalse{} = At(\Pi) - \gamma_{\Pi}(\itrue{})$.

Like $\GLop$, the ambiguity propagating ($\dlGLop$) and blocking ($\beta$) operators for defeasible theories are defined using reducts and an immediate consequence operator. $\beta$ is defined for all defeasible theories, but $\dlGLop$ is only defined for a restricted class.  We consider $\alpha$ first.

\begin{definition}Let $D = \langle R,C,\myemptyset\rangle$ be a defeasible theory such that $R_u = \myemptyset$.  If $S \subseteq Lit(D)$, then the $\dlGLop$-reduct $D_{\dlGLop}^S$ of $D$ \emph{wrt} $S$ is the set of rules $R_s \cup R^S_{d}$, where
\begin{center}
$R_d^S = \{r| r \in R_d$ and $(\forall c \in C[head(r)])(\exists q \in c-\{head(r)\})(q \notin S)\}$
\end{center}

\end{definition}

\begin{definition}Let $R$ be a set of strict and defeasible rules taken from $D$. If $S \subseteq Lit(D)$, then
\begin{center}
$T_{R}(S) =_{def} \{p|r \in R$ and $body(r) \subseteq S\}$.\end{center}
\end{definition}

\begin{definition}If $R$ is a set of strict and defeasible rules, the sequence $T_R\uparrow 0$, $T_R\uparrow 1$, $\ldots$ is:

\begin{enumerate}
\item $T_R\uparrow 0 = \myemptyset$
\item $T_R\uparrow\lambda+1 = T_R(T_R\uparrow\lambda)$ \emph{(for successor ordinals $\lambda+1$)}
\item $T_R\uparrow \lambda = \displaystyle\bigcup_{\kappa<\lambda}T_R\uparrow\kappa$ \emph{(for limit ordinals $\lambda$)}
\end{enumerate}
\end{definition}
As with logic programs, where $R$ is a set of (defeasible and strict rules) we define $Cl(R)$ as $T_R\uparrow\omega$.

\begin{definition}Let $D = \langle R,C,\myemptyset\rangle$ be a defeasible theory such that $R_u = \myemptyset$.  For any $S \subseteq Lit(D)$,  $\dlGLop_D(S) =_{def} Cl(D_{\dlGLop}^S)$.
\end{definition}

Under the translation of defeasible theories into logic programs, there is a correspondence between $\dlGLop$ and
$\GLop$. In order for the correspondence to hold, $C[p]$ is still required to be finite for
each $p \in Lit(D)$, and both $R_u$ and $\prec$ must be empty.
\begin{proposition}\label{equivGLops1}
 If $D = \langle R,C,\myemptyset\rangle$ is a defeasible theory such that $R_u = \myemptyset$ and $C[p]$ is finite for each $p \in Lit(D)$,
and if $\Pi$ is the logic program translation of $D$, then for any $S \subseteq Lit(D)$,

\begin{center}$\dlGLop_D(S) = \GLop_\Pi(S)$.\end{center}
\end{proposition}
\begin{proof*}
The proof proceeds by induction on the simple immediate consequence operator $T$ used to
compute the closure of reducts. Note that this operator is continuous and $T\uparrow \omega =
\displaystyle\bigcup_{n < \omega} T\uparrow n$, and so
it suffices to show that for each $n < \omega$,

\begin{center}$T_{\Pi^S}\uparrow n =  T_{D_{\dlGLop}^S}\uparrow n$.\end{center}

\noindent
The claim trivially holds for $n = 0$.
Suppose it holds for all $i < n$ and let $p \in T_{D_{\dlGLop}^S}\uparrow n$. Then there
is an $r \in R_{sd}[p]$ in the reduct ${D_{\dlGLop}^S}$ such that $body(r) \subseteq  T_{{D_{\dlGLop}^S}}\uparrow (n-1)$. If $r$ is strict, then there exists
an $r' \in \Pi$ such that $r' = r$ (ignoring notational differences; observe that $body(r) = body(r')$). Since $r'$ lacks default literals, $r' \in \Pi^S$. Since $body(r) \subseteq  T_{{D_{\dlGLop}^S}}\uparrow (n-1)$, by the
inductive hypothesis $body(r) \subseteq  T_{\Pi^S}\uparrow (n-1)$, and so by definition of the immediate consequence
operator $p \in T_{\Pi^S}\uparrow n$. If $r$ is defeasible, then for all $c \in C[p]$, there exists a $q \in c-\{p\}$ such that
$q \notin S$. As such, there exists a rule $r' \in trans(r)$ such that for each $\naf b$ in the body of $r'$ we have
$b \notin S$. As this is so, $r' \in  \Pi^S$
 and every default literal in the body of $r'$ has been deleted.
 And so $body(r) = body(r')$.
 Since $body(r) \subseteq  T_{{D_{\dlGLop}^S}}\uparrow (n-1)$, by the
inductive hypothesis $body(r) \subseteq  T_{\Pi^S}\uparrow (n-1)$.
 It follows that  $p \in T_{\Pi^S}\uparrow n$.

 Now suppose $p \in T_{\Pi^S}\uparrow n$. Then there is a rule $t \in \Pi^S$ such that $body(t) \subseteq  T_{\Pi^S}\uparrow (n-1)$. If $t$
corresponds to a strict rule of D, then $t \in {D_{\dlGLop}^S}$ and by inductive hypothesis $body(t) \subseteq  T_{{D_{\dlGLop}^S}}\uparrow (n-1)$
and so by definition of the immediate consequence operator $p\in  T_{{D_{\dlGLop}^S}}\uparrow n$. If $t$ corresponds to a
defeasible rule $t'$, since $t$ appears in $\Pi^S$, it must be the case that every default
literal of $t$ has been deleted. This means that for each conflict set $c \in C[p]$ there is an element
$q \in c-\{p\}$ such that $q \notin S$. Since this is so, then by definition of ${D_{\dlGLop}^S}$, $t' \in  {D_{\dlGLop}^S}$. By the inductive
hypothesis, $body(t) \subseteq  T_{{D_{\dlGLop}^S}}\uparrow (n-1)$ and so as before, $p \in  T_{{D_{\dlGLop}^S}}\uparrow n$.  \mathproofbox
\end{proof*}

Earlier, we showed a correspondence between the well-founded model for logic programs
and the well-founded model for ADL. This correspondence holds for theories with minimal conflict
sets and no defeaters or priorities on rules. Given the correspondence just shown between $\dlGLop$ and the GL-operator
$\GLop$, we can now see that $\dlGLop$ can be used to determine the consequences of ADL for these
theories. Note, however, that in Proposition \ref{equivGLops1}  conflict sets need not be minimal.  In this way, the operator defines a consequence relation that more generally corresponds to the WFS consequences than does ADL. Returning to Example \ref{ADLdiffersWFS}, the defeasible theory's well-founded model according  to $\dlGLop$ coincides with the well-founded model of the corresponding logic program. Both differ from the ADL-consequences of the theory.

As shown in \ref{elimDefeaters}, it is possible to transform a defeasible theory into an equivalent one in which defeaters, priorities on rules, and extended conflict sets do not appear. As this is so, $\dlGLop$ can in fact be used to compute the consequences of theories according to ADL. Nevertheless, this is not as satisfying as having an operator which more naturally corresponds to ADL, and we do not know at this point whether  $\dlGLop$ can be easily modified to serve this purpose.

Unlike $\dlGLop$, the blocking operator $\beta$ places no special restrictions on defeasible theories.
As shown below, the alternating fixpoint procedure defined with it can be used to compute the well-founded model according to NDL.

\begin{definition} Let $D = \langle R, C, \prec\rangle$ be a defeasible theory and $S
\subseteq Lit(D)$.   The $\beta$--reduct of $D$ \emph{wrt} $S$ (written $D_\bclos^S$) is $R_{s} \cup R_{d}^S$, where $R_d^S$ is the set of rules $r$ such that

\begin{itemize}
\item $r \in R_d$, and
\item $(\forall c \in C[head(r)])(\exists q \in c-\{head(r)\})(\forall s \in R[q])[body(s) \nsubseteq S$ or $s \prec r]$
\end{itemize}
\end{definition}
Here, no defeaters are  included in the reduct.

\begin{definition}
If $D$ is a defeasible theory and $S \subseteq Lit(D)$, then
 $\bclos_D{}(S) =_{def} Cl(D_{\bclos{}}^S)$.
\end{definition}

\begin{definition}
Let $D$ be a defeasible theory. We define the following sequence:
\begin{enumerate}
\item $X_D{}\uparrow 0 =_{def} \myemptyset$.
\item $X_D{}\uparrow \lambda + 1 =_{def} \bclos_{D}^2(X_D\uparrow \lambda)$ (for successor ordinals $\lambda + 1$).
\item $X_D{}\uparrow \lambda =_{def} \displaystyle\bigcup_{\kappa < \lambda} X_D{}\uparrow \kappa$ (for limit ordinals $\lambda$).
\end{enumerate}
\end{definition}

Below, since we will only use a single theory $D$, we will omit $D$ as a subscript, writing, e.g., $\beta(S)$ instead of $\beta_D(S)$.  Furthermore, we will omit the $\beta$ when writing the blocking reduct of $D$ wrt $S$, writing $D^S$ instead of $D_\beta^S$. The  $\alpha$--reduct is never used here, and so there will be no confusion.

\begin{proposition}\label{NDLop1}
Let $\tva{D,WF} = \langle \itrue{D,WF},
\ifalse{D,WF} \rangle$ be the \emph{wfm} of defeasible theory $D$ \emph{wrt} NDL. For
all $\lambda\geq 0$,
\begin{enumerate}
\item  if $p \in X_D\uparrow \lambda$, then $p \in \itrue{D,WF}$.
\item  if $p \notin \bclos(X_D\uparrow \lambda)$, then $p \in \ifalse{D,WF}$.
\end{enumerate}

\end{proposition}

\begin{proof*}
By definition of $X\uparrow 0$, $p \notin X\uparrow 0$ for any $p \in Lit(D)$. If
$p \notin \beta(X\uparrow 0)$ then it is impossible to
derive $p$ from the rules of $D$ under any circumstances, and so $p$ is in $\ifalse{D,WF}$. Suppose the hypothesis
holds for all $\kappa < \lambda$. We prove each case.

\begin{enumerate}
 \item Suppose $p \in X\uparrow \lambda$.  There must exist a least successor ordinal $\kappa \leq \lambda$ such that $p \in X\uparrow \kappa$. Recall that
\begin{displaymath}X\uparrow \kappa = \bclos(\bclos(X\uparrow {\kappa{}-1})) = Cl(D^{\bclos(X\uparrow {\kappa{}-1})}) = T_{D^{\bclos(X\uparrow {\kappa{}-1})}}\uparrow \omega.\end{displaymath}
Suppose that for all $i < m$, if $a \in T_{D^{\bclos(X\uparrow {\kappa{}-1})}}\uparrow i$,
then $a \in \itrue{D,WF}$  (this obviously holds for
$i = 0$). Let $p \in  T_{D^{\bclos(X\uparrow {\kappa{}-1})}}\uparrow m$.  Then there is an $r$ in
$D^{\bclos(X\uparrow {\kappa{}-1})}$ such that $body(r) \subseteq  T_{D^{\bclos(X\uparrow {\kappa{}-1})}}\uparrow (m-1)$. By inductive hypothesis, $body(r)
\subseteq \itrue{D,WF}$.

%If $r$ is strict, then obviously $p \in \itrue{D,WF}$ by definition of $T_D$ and $\tva{D,WF}$.
If $r$ is defeasible, then since $r \in D^{\bclos(X\uparrow {\kappa{}-1})}$, for all conflict sets $c \in C[p]$, there
exists a $q \in c - \{p\}$ such that for each rule $s \in R[q]$,
either $s \prec r$ or  $body(s) \nsubseteq  \bclos(X\uparrow {\kappa{}-1})$. If the latter, then by inductive hypothesis, $body(s) \cap  \ifalse{D,WF} \neq \myemptyset$.  So, there is a rule $r$ of $D$ such that $r$ is
strict and $body(r) \subseteq \itrue{D,WF}$, or else $r$ is defeasible,
$body(r) \subseteq \itrue{D,WF}$,  and for all conflict sets  $c \in
C[p]$, there exists a $q \in c - \{p\}$ such that for each rule $s
\in R[q]$, either $s \prec r$ or $body(s) \cap  \ifalse{D,WF} \neq \myemptyset$.  By definition of immediate
consequence in NDL and $\tva{D,WF}$, $p \in \itrue{D,WF}$.

\item  Now suppose $p \notin \bclos(X\uparrow {\lambda})$, and let  $A$ be the set of
elements not in $\bclos(X\uparrow {\lambda})$.  Let $r$ be a rule for $p$.  If
$r$ is strict, then $r$ is in the reduct of $D$ relative to
$X\uparrow \lambda$ and there is some $q \in body(r)$ such that $q \in A$ (this must be the case since $\bclos(X\uparrow {\lambda})$ is closed).
If $r$ is defeasible then $r$ is either in the reduct or not.  If it
is, then as before there is some $q \in body(r)$ such that $q \in
A$. If not, then there is a conflict set $c \in C[p]$ such that for
all $q \in c -\{p\}$, there is a rule $s \in R[q]$ such that
$body(s) \subseteq X\uparrow \lambda$ and $s \nprec r$.  From Case 1, if $body(s) \subseteq X\uparrow \lambda$ then  $body(s) \subseteq \itrue{D,WF}$.  Generalizing on $r$ and
then on $p$, for each $a \in A$ and each $r \in R_{D,sd}[a]$, either there exists a $q \in body(r)$ such that $q \in A$ or else $r \in R_d[a]$ and there exists a conflict set $c \in C[a]$ such that for each $v \in c - \{a\}$ there is a rule $s \in R[v]$ such that $body(s) \subseteq \itrue{D,WF}$ and $s \nprec r$.  It can be seen that
$A$ is unfounded under NDL with respect
to $D$ and $\tva{D,WF}$.  As such $A\subseteq U_D(\tva{D,WF}) \subseteq \ifalse{D,WF}$. Since $p \in A$, $p \in \ifalse{D,WF}$. \proofend
\end{enumerate}
\end{proof*}

\begin{proposition}\label{NDLop2}
Let  $D$ be a defeasible theory and $(\tva{D})$ the
sequence of interpretations defined for $D$ under the NDL-well-founded
semantics. For any $\lambda \geq 0$, there exists a $\eta \geq 0$ such that

\begin{enumerate}
\item if $p \in \itrue{D,\lambda}$, then $p \in X\uparrow\eta$.

\item if  $p \in
\ifalse{D,\lambda}$, then $p \notin
\bclos(X\uparrow\eta)$.
\end{enumerate}
\end{proposition}

\begin{proof*}
The hypothesis is trivially satisfied for $\lambda=0$.  Suppose it holds for all ordinals less than $\lambda$. We consider each case.

\begin{enumerate}
\item Suppose $p \in \itrue{\lambda}$.  Then
one of two cases applies:

\begin{enumerate}

\item There exists a rule $r \in R_s[p]$
such that $body(r) \subseteq \itrue{\kappa}$ for some successor ordinal $\kappa < \lambda$. If that is the case, then
by inductive hypothesis, there exists an $\eta \geq 0$ such that
$body(r) \subseteq X\uparrow \eta$.  Since $r$ is strict and $X\uparrow\eta$
is closed under strict rules, $p \in X\uparrow\eta$.

\item there exists a defeasible rule $r$ such that such
that $body(r) \subseteq \itrue{\kappa}$ for some $\kappa < \lambda$ and for all $c \in C[p]$ there
is a $q \in c -\{p\}$ such that for all rules $s \in R[q]$,  either
$s \prec r$ or else there exists a $v \in body(s)$ such that $v \in \ifalse{\kappa}$. If the latter, then by inductive hypothesis, there
exists a $\eta$ such that $v \notin \bclos(X\uparrow\eta)$. Since $body(r)
\subseteq \itrue{\kappa}$, then by inductive hypothesis, $body(r)
\subseteq X\uparrow\iota$ for some ordinal $\iota$.  Note that for any ordinals $\alpha$ and $\gamma$, if $\alpha < \gamma$, then $X\uparrow\alpha \subseteq X\uparrow\gamma$ and $\bclos(X\uparrow\gamma) \subseteq \bclos(X\uparrow\alpha)$, and so for any literal $b$, if $b \notin \bclos(X\uparrow\alpha)$, then for all $\gamma > \alpha$ it follows that $b \notin \bclos(X\uparrow\gamma)$.
With that in mind,
generalizing on $s$ and then $c$, and letting $\iota'$ be the least ordinal such that $\eta<\iota'$ and $\iota < \iota'$ for any of the above $\iota$'s, we have $body(r) \subseteq X\uparrow \iota'$ and
for all $c \in C[p]$ there is a $q \in c -\{p\}$ such that for each $s \in R[q]$, $body(s) \nsubseteq \bclos(X\uparrow \iota')$ or $s \prec r$.  As such $r \in D^{\bclos(X\uparrow \iota')}$. Since $body(r) \subseteq X\uparrow {\iota'}$, by monotonicity we have
$body(r) \subseteq X\uparrow {\iota'+1}$.  Recall that $X\uparrow {\iota'+1} = \bclos(\bclos(X\uparrow \iota')) = Cl(D^{\bclos(X\uparrow \iota')})$.
We thus have $r \in D^{\bclos(X\uparrow \iota')}$ and $body(r) \subseteq Cl(D^{\bclos(X\uparrow \iota')})$.  From this it follows that
$p \in
Cl(D^{\bclos(X\uparrow \iota')})$, i.e. $p \in X\uparrow {\iota'+1}$.

\end{enumerate}

\item Suppose  $p \in \ifalse{\lambda}$.  If $\lambda$ is a limit ordinal, then there exists some successor ordinal $\kappa < \lambda$ such that $p \in \ifalse{\kappa}$. By inductive hypothesis, there exists some $\eta \geq 0$ such that $p \notin
\bclos(X\uparrow\eta)$. So suppose $\lambda$ is a successor ordinal.  By definition, $\ifalse{\lambda}$ is
an unfounded set \emph{wrt} to $D$ and $\tva{\lambda-1}$.

Let $a \in \ifalse{\lambda}$ and $r \in R_{sd}[a]$. As such, either (1) there
is a $v \in body(r)$ such that $v \in \ifalse{\lambda} \cup \ifalse{\lambda-1}$ (which by monotonicity of $U_D$ means $v \in \ifalse{\lambda}$)  or else (2) $r \in R_d[a]$ and there exists a conflict set $c \in
C[a]$ such that for each $q \in c-\{a\}$, there is a rule $s \in
R[q]$ such that $body(s) \subseteq \itrue{\lambda-1}$ and $s \nprec r$. Suppose (2) holds. Since
$body(s) \subseteq \itrue{\lambda-1}$, then by inductive
hypothesis, there exists a $\gamma \geq 0$ such that $body(s) \subseteq
X\uparrow \gamma$. Generalizing on $q$, there exists a $\eta \geq 0$ such that for each $q \in c -\{a\}$ there exists a $s \in R[q]$ such that $body(s) \subseteq X\uparrow \eta$ and $s \nprec r$. As this is so, by definition of reduct for NDL,  $r \notin D^{X\uparrow \eta}$.
Thus, if  $r \in D^{X\uparrow \eta}$, then $body(r) \cap  \ifalse{\lambda} \neq \myemptyset$.
Generalizing on $r$ and then $a$, we may conclude that for each $v \in \ifalse{\lambda}$ and each $r \in R_{sd}[v]$, if  $r \in D^{X\uparrow \eta}$, then $body(r) \cap  \ifalse{\lambda} \neq \myemptyset$.

Suppose for a proof by contradiction that $\ifalse{\lambda} \cap Cl(D_\bclos^{X\uparrow \eta}) \neq \myemptyset$. Then there is a least integer $i > 0$ such that
$\ifalse{\lambda}  \cap T_{D^{X\uparrow \eta}}\uparrow i \neq \myemptyset$.
Let $v \in \ifalse{\lambda}  \cap T_{D^{X\uparrow \eta}}\uparrow i$. Since $v \in T_{D^{X\uparrow \eta}}\uparrow i$, it follows that there exists a rule $r \in R_{sd}[v]$ such that $r \in D^{X\uparrow \eta}$ and $body(r) \subseteq T_{D^{X\uparrow \eta}}\uparrow(i-1)$. However, since $r \in D^{X\uparrow \eta}$, it must be that  $body(r) \cap \ifalse{\lambda} \neq \myemptyset$. Thus $\ifalse{\lambda}  \cap T_{D^{X\uparrow \eta}}\uparrow i-1 \neq \myemptyset$.  This is a contradiction, and so $\ifalse{\lambda} \cap Cl(D_\bclos^{X\uparrow \eta}) = \myemptyset$.

Since $p \in \ifalse{\lambda}$ and $Cl(D_\bclos^{X\uparrow \eta}) = \bclos(X\uparrow \eta + 1)$, it follows that $p \notin
\bclos(X\uparrow \eta + 1))$.\proofend
\end{enumerate}\end{proof*}

From the above propositions, a correspondence between the sequence $X\uparrow 0$, $X\uparrow 1$, $\ldots$, and the well-founded model according to NDL is
established.

\begin{proposition} If $D$ is a defeasible theory, $\tva{D,WF}$ its well-founded model according to NDL, and $\lambda$ the closure ordinal of the sequence  $X_D\uparrow0$, $X_D\uparrow1$, $\ldots$, then $$\tva{D,WF} = \langle X_D\uparrow\lambda,  Lit(D) - \beta_D(X_D\uparrow\lambda)\rangle.$$
\end{proposition}

\section{Stable Sets for Defeasible Theories}\label{stableModelsForDL}

NDL and ADL, like the well-founded semantics for  logic programs,  are \emph{directly skeptical} formalisms.  If a literal $p$ is a consequence of a theory, then there must be some rule for it with a body that is also a consequence of the theory. This is in contrast to \emph{indirectly skeptical} formalisms, such as  default logic \cite{reiter:1} and the stable model and answer--set  \cite{gelfond91} semantics for logic programs, where  consequences are defined indirectly via the intersection of \emph{extensions} (stable models, answer--sets).
These formalisms allow  {\em floating conclusions} \cite{makinson-schlecta}---i.e. consequences which appear in every extension but which have no support appearing in every extension. Directly skeptical formalisms do not allow floating conclusions.

\begin{example}[Ginsberg's extended Nixon Diamond]\label{NixonDiamond}$D = \langle R, C_{MIN}\cup\{\{dove,hawk\}\}, \varnothing\rangle$, where $R$ is
\end{example}

\begin{center}
\begin{tabular}{ll}
1. $\varnothing \rightarrow nixon$ & \\
2.  $\{nixon\} \rightarrow republican$ &
3.  $\{nixon\} \rightarrow quaker$  \\
4.  $\{quaker\} \Rightarrow dove$ &
5.  $\{republican\} \Rightarrow hawk$ \\
6. $\{hawk\} \rightarrow \neg dove$&
 7.  $\{dove\} \rightarrow \neg hawk$\\
8. $\{hawk\} \Rightarrow extremist$&
9. $\{dove\} \Rightarrow extremist$\\
\end{tabular}
\end{center}
The logic program counterpart to the above theory is
\begin{center}
\begin{tabular}{ll}
1. $nixon$ & \\
2. $republican \neck nixon$ &
3. $quaker \neck  nixon$  \\
4. $dove \neck \naf \neg dove, \naf hawk, quaker$ &
5. $hawk \neck \naf \neg hawk, \naf dove, republican$ \\
6. $\neg dove \neck hawk$&
7. $\neg hawk \neck dove$\\
8. $extremist \neck \naf \neg extremist, hawk$&
9. $extremist \neck \naf \neg extremist, dove$\\
\end{tabular}
\end{center}
Here, the positive ADL-- and NDL-- consequences of the theory  agree with the well-founded model of the logic program: $nixon$, $republican$, and $quaker$ are all well-founded, but no other literal is. In ADL, the literals $dove$, $hawk$, and $extremist$ are all ambiguous. They are unfounded in NDL.

The logic program has two stable models, where $S$ is a stable model of $\Pi$ if $\GLop_\Pi(S) = S$.

\begin{center}
\begin{tabular}{l}
$S_1$: $\{nixon, republican, quaker, dove, \neg hawk, extremist\}$\\
$S_2$: $\{nixon, republican, quaker, \neg dove, hawk, extremist\}$ \\
\end{tabular}
\end{center}
Since $extremist$ appears in each such model, it is taken as a consequence of the program according to the stable model semantics. Since neither $dove$ nor $hawk$ appears in both models,  $extremist$ is a floating conclusion.

It is indeed possible to use both $\dlGLop$ and $\beta$ to define indirectly skeptical semantics similar to the stable model semantics for logic programs.  We do that here. As before, the semantics based on $\dlGLop$ only applies to a restricted class of defeasible theories.

\begin{definition}
Let $D = \langle R,C,\prec\rangle$ be a defeasible theory and $S \subseteq Lit(D)$.
\begin{enumerate}
\item If $\prec = \myemptyset$ and $R_u = \myemptyset$, then $S$ is an $\dlGLop$-\emph{stable set} of $D$ iff $S = \dlGLop_D(S)$.
\item $S$ is a $\beta$-\emph{stable set} of $D$ iff $S = \beta_D(S)$.
\end{enumerate}
\end{definition}
\begin{definition}Let $D = \langle R,C,\prec\rangle$ be a defeasible theory and $p \in Lit(D)$.
\begin{enumerate}
\item If $\prec = \myemptyset$ and $R_u = \myemptyset$,
\begin{enumerate}
    \item $D \dsmodels_{\dlGLop} p$ \emph{iff} $p\in S$ for all $\dlGLop$-stable sets $S$.
    \item $D \dsrefutes_{\dlGLop} p$ \emph{iff} $p\notin S$ for all $\dlGLop$-stable sets $S$.
\end{enumerate}
\item For arbitrary theories $D$,
\begin{enumerate}
    \item $D \dsmodels_{\beta} p$ \emph{iff} $p\in S$ for all $\beta$-stable sets $S$.
    \item $D \dsrefutes_{\beta} p$ \emph{iff} $p\notin S$ for all $\beta$-stable sets $S$.
\end{enumerate}
\end{enumerate}
\end{definition}

In Example \ref{NixonDiamond}, $D$ has two $\dlGLop$--stable sets, and these correspond to $S_1$ and $S_2$ above. As such,
$D \dsmodels_{\dlGLop} extremist$.  There is only one $\beta$--stable set, however:
\begin{center}
$\{nixon, republican, quaker\}$.
\end{center}
These three literals must appear in any $\beta$--stable set. However, this implies that rules  4 and 5 can appear in no $\beta$--reduct of $D$, and so $extremist$ can appear in no $\beta$--stable set. It is not a floating conclusion according to the semantics based on $\beta$.

As in the case for logic programs, the well-founded models according to ADL and NDL, respectively, are contained within the stable sets defined using $\dlGLop$ and $\beta$.

\begin{proposition}\label{equivGLops}
Let $D = \langle R,C,\prec\rangle$ be a defeasible theory.
  \begin{enumerate}

\item  If $\prec = \myemptyset$, $R_u = \myemptyset$,  and $\langle \itrue{}, \ifalse{}\rangle$ is the well-founded model of $D$ according to ADL, then   for any $\dlGLop$-stable set $M$ of $D$,
 $\itrue{} \subseteq M$, and
 $\ifalse{} \cap \dlGLop_D(M) = \myemptyset$.

\item  If $\langle \itrue{}, \ifalse{}\rangle$ is the well-founded model of $D$ according to NDL, then   for any $\beta$-stable set $M$ of $D$,
 $\itrue{} \subseteq M$, and
 $\ifalse{} \cap \beta_D(M) = \myemptyset$.

\end{enumerate}

\end{proposition}

\begin{proof}The proofs are the same for both ADL and NDL, and so we consider only $\dlGLop_D$. Define $X\uparrow0$, $X\uparrow1$, etc., as above. Clearly, since $X\uparrow0 = \myemptyset$, $X\uparrow0 \subseteq M$. Suppose $X\uparrow\kappa \subseteq M$ for each $\kappa < \lambda$. We may assume \emph{wlog} that $\lambda$ is a successor ordinal. Observe that $X\uparrow\lambda =  \dlGLop_{D}^2(X\uparrow{\lambda-1})$. By inductive hypothesis, $X\uparrow{\lambda-1} \subseteq M$. Since $\dlGLop_{D}^2$ is monotone,  $\dlGLop_{D}^2(X\uparrow{\lambda-1}) \subseteq \dlGLop_{D}^2(M)$.  Since $X\uparrow\lambda =  \dlGLop_{D}^2(X\uparrow{\lambda-1})$ and $M$ is a stable set, it follows that $X\uparrow{\lambda} \subseteq M$. Generalizing, $\itrue{} \subseteq M$. Since $\dlGLop_{D}$ is antimonotone, it follows that $\dlGLop_{D}(M) \subseteq \dlGLop_{D}(\itrue{})$, and so
  $Lit(D) -  \dlGLop_{D}(\itrue{}) \subseteq Lit(D) - \dlGLop_{D}(M)$. I.e.,   $\ifalse{} \subseteq Lit(D) - M$. And so
  $\ifalse{} \cap M = \myemptyset$.\end{proof}

No stable set of a defeasible theory is a subset of another.  This parallels the case for the stable models/answer-sets of logic programs \cite{gelfond91} and is the result of $\alpha$ and $\beta$ being antimonotone. For instance, if $S_1$ and $S_2$ are $\dlGLop$--stable sets such that $S_1 \subseteq S_2$, then $\alpha_D(S_2) \subseteq \alpha_D(S_1)$, and so $S_2 \subseteq S_1$.

\begin{proposition}
If $S_1$ and $S_2$ are distinct $\alpha$ ($\beta$) stable sets of defeasible theory $D$, then $S_1 \not\subseteq S_2$.
\end{proposition}

Given the close connection between $\dlGLop$ and $\gamma$, for each defeasible theory $D$ with no defeaters and priorities, then provided that $C[p]$ is finite for each $p \in Lit(D)$, we may conclude that the $\dlGLop$-stable sets of $D$ correspond to the stable models of its logic program  translation. Furthermore,  as shown below (Propositions \ref{stablemodels1}--\ref{LP2DL2}), a correspondence for the translation in the reverse direction also holds. That is, if $\Pi$ is a normal logic program, then the stable models of $\Pi$ correspond to the $\alpha$-stable sets of $D_{\Pi}$.

Observe that this implies that any stable model of $\Phi$ is a classical interpretation (and so can be represented as a set of atoms).
Below, if $X \subseteq At(\Pi)$, let

\begin{center}
$X^{\neg} = X \cup \{\neg p|p \in At(\Pi)$ and $p \notin X\}$.
\end{center}

\begin{proposition}\label{stablemodels1}
Let $\Pi$ be a normal logic program and $\Phi$ the explicit version of $\Pi$. For any $M \subseteq At(\Pi)$, $T_{\Pi^M}\uparrow \omega = (T_{\Phi^{M^{\neg}}}\uparrow \omega \cap At(\Pi))$.
\end{proposition}

\begin{proof}We show that for all $i \geq 0$ and $p \in At(\Pi)$, $p \in T_{\Pi^M}\uparrow i$ implies $p \in T_{\Phi^{M^{\neg}}}\uparrow \omega$, and $p \in T_{\Phi^{M^{\neg}}}\uparrow i$ implies $p \in T_{\Pi^M}\uparrow \omega$. The case for $i = 0$ is vacuous. Suppose the claim holds for all $i<n$.

If $p \in T_{\Pi^M}\uparrow n$, then there is a rule $r \in\Pi$ such that $body(r)^{+}\subseteq  T_{\Pi^M}\uparrow (n-1)$ and   $q \notin M$ for each $q \in body(r)^{-}$. Let $r' \in \Phi$ be strict rule corresponding to $r$. By inductive hypothesis, $body(r)^{+}  \subseteq T_{\Phi^{M^{\neg}}}\uparrow \omega$.  For each $q$, $\neg q\neck \myemptyset\in\Phi^{M^{\neg}}$ and so $\neg q \in T_{\Phi^{M^{\neg}}}\uparrow \omega$. It follows that $body(r') \subseteq T_{\Phi^{M^{\neg}}}\uparrow \omega$, and so $p \in T_{\Phi^{M^{\neg}}}\uparrow \omega$.

If $p \in T_{\Phi^{M^{\neg}}}\uparrow n$, then there is a rule $r\in\Phi^{M^{\neg}}$ such that $body(r) \subseteq T_{\Phi^{M^{\neg}}}\uparrow (n-1)$. By inductive hypothesis, $a  \in T_{\Pi^{M}}\uparrow \omega$ for each atom $a \in body(r)$.  For each  $\neg q \in body(r)$, it must be that $\neg q \neck \myemptyset \in\Phi^{M^{\neg}}$, and so $q \notin M^{\neg}$ and  $q \notin M$. Rule $r$ corresponds to a rule $r' \in \Pi$ such that atom $a \in body(r)$ \emph{iff} $a \in body(r')^{+}$, and  $\neg q \in body(r)$ \emph{iff} $q \in body(r')^{-}$. Given that no $q\in body(r')^{-}$ appears in $M$,  $p \neck body(r')^{+}\in \Pi^{M}$.  Since $body(r')^{+}  \subseteq T_{\Pi^{M}}\uparrow \omega$, it follows that $p \in T_{\Pi^{M}}\uparrow \omega$.\end{proof}

\begin{proposition}\label{stablemodels2}
Let $\Pi$ be a normal logic program and $\Phi$ its explicit version. $M\subseteq At(\Pi)$ is a stable model of $\Pi$ \emph{iff} $M^\neg$ is a stable model of $\Phi$.
\end{proposition}

\begin{proof}
$M^\neg \cap At(\Pi) = M$, and from  Prop. \ref{stablemodels1}, $T_{\Pi^M}\uparrow \omega = (T_{\Phi^{M^\neg}}\uparrow \omega \cap At(\Pi))$.
If $M$ is a stable model of $\Pi$, $M =  (T_{\Phi^{M^\neg}}\uparrow \omega \cap At(\Pi))$, and so for each  $p\in At(\Pi)$, $p \in M^\neg$ \emph{iff} $p \in T_{\Phi^{M^\neg}}\uparrow\omega$.  If $\neg p \in T_{\Phi^{M^\neg}}\uparrow \omega$, then $\neg p \neck \myemptyset\in\Phi^{M^\neg}$ and so $p \notin M^\neg$. If that is so, then $p \notin M$ and (by definition of $M^\neg$)  $\neg p \in M^\neg$. Conversely, if
$\neg p \in M^\neg$, then  $p \notin M$, and so $\neg p \neck \myemptyset\in\Phi^{M^\neg}$. Consequently, $\neg p \in T_{\Phi^{M^\neg}}\uparrow \omega$.  As such, $M^\neg$ is a stable model of $\Phi$.
If, in turn, $M^\neg$ is a stable model of $\Phi$,  $T_{\Pi^M}\uparrow \omega =  (M^\neg \cap At(\Pi)) = M$, and so $M$ is a stable model of $\Pi$. \end{proof}

\begin{proposition}\label{LP2DL2}
Let $\Pi$ be a normal logic program and $D_\Pi$ its defeasible logic translation.  $M$ is a stable model of $\Pi$ \emph{iff} $M^\neg$ is an $\dlGLop$-stable set of $D_\Pi$.
\end{proposition}
\begin{proof*}Translating $D_\Pi$ into a logic program using the Brewka-inspired scheme yields $\Phi$. As implied by Proposition \ref{equivGLops1}, the $\dlGLop$-stable sets of $D_\Pi$ correspond to the stable models of $\Phi$.  However, by Prop. \ref{stablemodels2}, there is a 1-1 correspondence between the stable models of  $\Phi$ and those of $\Pi$.
\mathproofbox
\end{proof*}

\begin{example}\label{ambpropex}The defeasible theory from Example \ref{ambiex} and its logic program translation are shown again below.

\begin{center}
\begin{tabular}{cc}
\begin{minipage}{.4\textwidth}
\begin{enumerate}
\item  $\myemptyset \Rightarrow p$
\item  $\myemptyset \Rightarrow \neg p$
\item  $\{p\} \Rightarrow \neg q$
\item  $\myemptyset \Rightarrow q$
\end{enumerate}\end{minipage}
 &
\begin{minipage}{.40\textwidth}
\begin{enumerate}
\item $p \neck \naf \neg p$
\item $\neg p \neck \naf p$
\item $\neg q \neck \naf q, p$
\item $q \neck \naf \neg q$
\end{enumerate}
\end{minipage}
\end{tabular}
\end{center}
The $\dlGLop$-stable sets of the defeasible theory are $\{p,q\}$,  $\{p,\neg q\}$, and $\{\neg p, q\}$. These are also the stable models of the counterpart logic program. The only $\beta$-stable set is $\{q\}$, however. Neither $p$ nor $\neg p$ can appear in any stable set (the rules for them would be deleted in any $\beta$-reduct), and this implies that $\neg q$ cannot appear, either.
\end{example}

As reported earlier \cite{maier2009}, ADL is more conservative than NDL, in the sense that for all $D$ and $p$,  if $D \dsmodels_{ADL} p$, then $D \dsmodels_{NDL} p$.
From this, it readily follows that  $D \dsmodels_{ADL} p$ implies $D \dsmodels_{\beta} p$.
However, the similar claim does not hold if $\dsmodels_\alpha$ is used in place of $\dsmodels_{ADL}$. That is, $D \dsmodels_{\alpha} p$ does not  imply $D \dsmodels_{\beta} p$.

\begin{example}$D = \langle R, C_{MIN}, \varnothing\rangle$, where $R$ is

\begin{enumerate}
\item $\myemptyset \Rightarrow p$
\item $\myemptyset \Rightarrow \neg p$
\item $\{p\} \Rightarrow q$
\item $\{\neg p\} \Rightarrow q$
 \end{enumerate}
\end{example}
Here, the $\alpha$-stable sets are $\{p,q\}$ and $\{\neg p, q\}$, and so  $D \dsmodels_\alpha q$. However, the only $\beta$-stable set is $\myemptyset$, which implies $D \dsrefutes_\beta q$.

ADL is also more conservative than NDL in the sense that $D \dsrefutes_{ADL} p$ implies $D \dsrefutes_{NDL} p$.   We don't know yet whether $D \dsrefutes_{\alpha} p$ implies $D \dsrefutes_{\beta} p$.

\section{Related Work}\label{relatedwork}

As stated above, NDL \cite{nute1999,nute2001,donnelly} and its ambiguity propagating counterpart ADL \cite{maier2006} were the first defeasible logics to incorporate failure-by-looping, and this cycle check essentially requires the proof-systems to be tree-based---different branches of computation must be kept distinct. In  BDL \cite{billington1993} and most other variants of defeasible logic, proofs are linear sequences of tagged literals. In these logics, cycles cannot be detected, and this affects the conclusions they can draw.  Maher and Governatori \shortcite{maher99}, however, do provide a well-founded semantics for BDL which correctly handles cycles. Presumably, the BDL proof system is sound but not complete relative to this semantics.

The logics based on BDL also differ from NDL and ADL in that they make a distinction between \emph{strict} and \emph{defeasible} derivations.   E.g.,  the expression $\definitely p $  in a derivation indicates that $p$ is derivable using only the strict rules of a theory, while $\bdldefeasibly p$ means that $p$ is derivable using the theory as a whole (the corresponding negative expressions $\definitelyno p$ and $\bdldefeasiblyno p$ indicate that $p$ is refutable).  Significantly,  if the body of a strict rule $r$ is only defeasibly derivable, then the rule is treated as a defeasible rule, i.e. a rule which can be defeated.   This prevents BDL from inferring contradictions except for those due to strict rules alone.

\begin{example}\label{intuitions}\mbox{$D = \langle R, C_{MIN}, \myemptyset\rangle$, $R$ is }

\begin{center}
\begin{minipage}{0.9\textwidth}
\begin{tabular}{lll}
1. $\myemptyset\Rightarrow married$ &
2. $\{married\} \rightarrow \neg bachelor$ &
3. $\myemptyset \Rightarrow bachelor$\\
\end{tabular}
\end{minipage}
\end{center}
\end{example}

In BDL and its variants, $married$ and
$bachelor$ do not conflict, and so $married$  is  defeasibly derivable (there is a proof ending in $\bdldefeasibly married$).  However, since the body of rule 2 is only defeasibly derivable, rule 2 is considered defeasible. Since BDL blocks ambiguity, in that logic both $bachelor$ and $\neg bachelor$ are defeasibly refuted (in the ambiguity propagating logic described by Antoniou et al. \citeyear{antoniou1}, both literals are ambiguous).
 In contrast, if conflict sets are closed under strict rules, then NDL and ADL hold that $married$ and $bachelor$
conflict and refrain from deriving either (they are refuted in NDL and ambiguous in ADL).  Antoniou  \citeyear{antoniou2006} calls the approach taken in ADL and NDL the
``purist view'', and he defends the alternative.
Brewka \citeyear{brewka} rejects the dual treatment of strict rules, however:
Strict rules are used to specify definitions, necessary relationships, etc.
To treat them sometimes defeasibly undermines this.  Essentially the same argument was made when the semantics for ADL and NDL was first developed \cite{maier2009}.
Extended conflict sets were introduced in NDL and ADL to avoid drawing inconsistent conclusions based on defeasible rules while at the same time maintaining the monolithic nature of strict rules.
%{Nute \citeyear{nute1994,nute1997} cites Schurtz \citeyear{schurtz} as the motivation for extended conflict sets.}

The first ambiguity propagating defeasible logics appeared around the
year 2000  \cite{antoniou2000}. Up to that
point, all defeasible logics were ambiguity blocking. In
\cite{antoniou2000}, the basic propagating logic---based on BDL---is
 presented as a system embedded in a logic program.  A formal proof system appeared separately
\cite{antoniou1,antoniou2001b}.

BDL itself extends an earlier logic
\cite{nute1989,billington1990}. Specifically, BDL adds variables and function symbols to the logic, and (importantly) it  allows the precedence relation to range over both strict and defeasible rules.  In the earlier
logic (as in NDL and ADL), strict rules are superior to all defeasible rules and no
strict rule is superior to any other strict rule. In his analysis, Brewka \citeyear{brewka} shows that when the
precedence relation is restricted to defeasible rules, the defeasible logic is sound but not
complete \emph{wrt} his prioritized well-founded semantics.

In a separate line of work, David Billington has developed a family of formalisms that are generally called \emph{plausible logic} (Billington \citeyearNP{Billington04,billington2005a,billington2005b,Billington08};
Billington and Rock \citeyearNP{Billington01}).  Plausible logic is based on defeasible logic, using both strict and defeasible rules, but it expands it to handle arbitrary clauses. Unlike in defeasible logic, disjunctions can be proved.  Extended conflict sets are not used, but the logics have what is called the \emph{general conflict property} \cite{Billington08}, meaning that defeasible rules conflict if they cannot all fire without contradicting the strict part of the theory. Proofs are again sequences of tagged formulas, and these tags are used to define multiple consequence relations (which correspond to different levels of certainty). Through the use of tags, the proof system simultaneously allows both the blocking and propagation of ambiguity. Loop detection is discussed in (Billington \citeyearNP{Billington04,Billington08}).  %A comparison of plausible logic, BDL, NDL/ADL, as well as a later clausal defeasible logic created by Billington, appears in \cite{Billington08}.
Given the number of NMR formalisms in existence today and the differing intuitions they embody, a formalism such as plausible logic---which attempts to unify these intuitions into a single system---appears very attractive.

In \cite{billington2007}, multiple semantics for plausible theories are provided, corresponding to differing intuitions  about acceptable consequences. Plausible theories are related to default theories \cite{reiter:1}, and it is shown how the framework can provide an ambiguity blocking semantics for default logic. Given the known relationships between default logic and the stable model semantics for logic programs \cite{marek1989}, the work in \cite{billington2007} can be seen as applying to logic programs.

Other variants of defeasible logic have been related to different NMR formalisms. A Dung-like
argumentation semantics for BDL and its variants is provided in \cite{governatori2000} and \cite{governatori2004}.  The relationship between defeasible logic (again, an ambiguity propagating variant of BDL)  and default logic is addressed in \cite{antoniou2001b}. A means of translating defeasible theories into default
theories is given, and it is shown that every defeasible consequence appears in every extension of the corresponding default theory. The
paper does not address refutations---i.e., it is not proven whether a literal defeasibly refuted is absent from every default extension.

The logic-programming embedding used by Antoniou et al. first appeared in \cite{maher99}.  It is shown there that the BDL-consequences of a defeasible theory correspond to those of the counterpart  program  under the Kunen semantics \cite{kunen}. The same paper presents the well-founded semantics for BDL mentioned above and shows that the consequences  under this semantics correspond to the well-founded model of the program.  In
\cite{antoniou05}, it is shown that under the translation, the
conclusions of the defeasible theory correspond to the intersection
of stable models of the program. This result holds only for what the authors call
\emph{decisive} theories---theories in which every literal is either
provable or refutable (or, equivalently, theories whose dependency graph is acyclic). Without decisiveness, the correspondence
holds only in one direction: every literal provable in the defeasible
logic appears in the intersection of stable models.

We note that the translation used by Antoniou et al. is not at all like the Brewka-inspired scheme described above, and in our opinion it does not by itself expose a close relationship between defeasible logic and logic programming. In their method of translation, the defeasible logic proof system is explicitly encoded in the logic program.  E.g., the proof-conditions governing strict derivations are represented (in Prolog notation) as

\begin{verbatim}
     definitely(X):-
          fact(X).

     definitely(X):-
          strict(RuleID,X, [Y1, ... , Yl]),
          definitely(Y1), ... , definitely(Yn).
\end{verbatim}

\noindent{}A statement $X$ is definitely (strictly) derivable ($\definitely X$) if $X$ is a fact of the theory, or if there is a strict rule with head $X$ and every literal of the body is also definitely derivable.  The rules of a defeasible theory are represented as facts in the logic program. E.g.,

\begin{verbatim}
     strict(rule1, bird(a), [swan(a)]).
     defeasible(rule2, white(a), [swan(a)]).
\end{verbatim}
In this fashion, the logic program encodes both the defeasible theory (as terms appearing in facts and rules) and the proof system itself. In the Brewka-inspired scheme, it is only the defeasible theory that is translated and not the entirety of the proof system.  Because of this, we  consider the relationships between ADL and the WFS, proven above, to be more insightful.

The fixpoint semantics for NDL and ADL appear in \cite{maier2009}, and it is shown there that the proof systems for NDL and ADL are sound with respect to their counterpart semantics, and that they are complete for locally finite theories.  It is also shown there that, when the priorities on rules are transitive, ADL and NDL  satisfy versions of Cut and Cautious Monotony (that is, they are cumulative). It is widely accepted that a good  nonmonotonic formalism should satisfy these.

Defeasible theories such as the one shown in Example \ref{ADLdiffersWFS} are problematic for both ADL and NDL.  In that example, $q$ is not well-founded in either NDL or ADL, but it intuitively should be  (in the corresponding logic program, $q$ is indeed well-founded). Examples such as this show that, while extended conflict sets are needed in some cases to draw reasonable conclusions, their use can cause problems in other cases. An alternative to using extended conflict sets is to keep conflict sets minimal while adding all possible transpositions of strict rules to the defeasible theory. If this is done, then the intuitively correct result can be drawn in Example \ref{ADLdiffersWFS}.  Doing this (or else closing conflict sets under strict rules), allows NDL and ADL to satisfy Consistency Preservation. That is, the logics cannot be used to derive contradictions that do not follow from the strict rules alone.  This is shown in \cite{maier2009}.

\section{Conclusion}

Nute's logic NDL was developed in isolation of the well-founded semantics, but the desire to handle theories containing cycles appears to be the same.
While it is unsurprising that the consequences under NDL do not correspond to those of the WFS---NDL blocks ambiguity while the WFS propagates it---we have shown here that under  natural translations of defeasible theories into logic programs (and vice versa), the consequences according to ADL and the WFS actually do coincide. This, in a sense, is surprising, as ADL was developed by making only a minor modification to the proof system of NDL.

The present research was initiated with an eye toward practicality. The ability to translate defeasible theories into logic programs means that existing logic programming systems
 can be used to reason according to ADL.  In the other direction,  NDL indirectly provides an ambiguity blocking semantics for logic programs, and ADL provides a representation of logic programs under the WFS that in some cases is intuitively easier to comprehend (this is an arguable point; nevertheless,  we suppose that some at least will find $\Rightarrow$ more readily understood than default negation).

The antimonotone operator defined for ADL only works properly when defeaters are not present in the defeasible theory and when the priority relation over rules is empty. Both defeaters and priorities can in fact be compiled way, however. That is, a defeasible theory $D$ of ADL or NDL can be transformed into an equivalent one $E$ such that $R_u = \myemptyset$, $\prec_{} = \myemptyset$, and $C$ is minimal. This is shown in \ref{elimDefeaters}.  A similar transformation is discussed in the context of other defeasible logics in \cite{antoniou2001}.  Nevertheless, while the elimination of priorities and defeaters allows the use of $\dlGLop$ to compute all of the ADL consequences of a theory, it is not a very satisfying solution, as it requires expressing important elements of the logic (e.g., conflicts, priorities) directly in the rules of the theory. In that sense, the transformation is similar to the embedding noted above of BDL  into logic programs.
  An operator which does not require any sort of transformation in order to do its work would be far better.  %This is a natural subject for future research.

Similar work on adding priorities to the WFS has been performed, notably by Brewka \citeyear{brewka96}, and also by
Torsten Schaub and Kewen Wang \citeyear{schaub2001}. Both have  developed prioritized well-founded semantics for extended logic programs, and in both cases, the models can be computed in polynomial time relative to the size of the program. At this point, we don't know how ADL relates to these formalisms,  and we haven't investigated whether their way of handling preferences can be easily adopted for use with ADL (or other defeasible logics).  It is certainly the case, however, that the two logic programming formalisms yield  results different than ADL, for the simple reason that both formalisms are explosive. E.g., in both formalisms, the well-founded model of the program

\begin{enumerate}
\item $p$
\item $\neg p$
\item $q \neck r, s, t$
 \end{enumerate}
is the set of all literals.  In contrast, $q$ would be considered unfounded according to ADL. In our view, this is the correct conclusion, as we really have no reason to believe $q$.  Other varieties of defeasible logic would similarly consider $q$ unfounded; none would conclude $q$. By their  nature, defeasible logics are paraconsistent.

\appendix

\section{NDL and ADL proof systems}\label{PS}

Proofs in NDL and ADL form argument trees, with nodes labeled with tagged literals (for a given node $n$, $label(n)$ refers to the label of $n$). In earlier defeasible logics, such as BDL, proofs are linear sequences of tagged literals.

\begin{definition}
Let $D$ be a defeasible theory. A
\emph{defeasible argument tree for} $D$ is a finite tree $\tau$ such that
every node of $\tau$ is labeled with one of $+p$ or $-p$, where $p$ is any literal in $Lit(D)$. If $\tau$ is a defeasible argument tree for $D$ and $n$ is a node in $\tau$, then $\tau$ is a \emph{positive} node iff $n$ is labeled $+p$, and $n$ is a \emph{negative} node iff $n$ is labeled $-p$.
\end{definition}

\begin{definition}Let $A$ be a set of
literals, and $n$ a node of a defeasible argument tree $\tau$.
    \begin{itemize}
    \item[1.]  $A$ \emph{succeeds} at $n$ iff for all $q \in A$,
    there is a child  of $n$ labeled $+q$.
       \item[2.]  $A$ \emph{fails} at $n$ iff there is a $q \in A$
       and a child  of $n$ labeled $-q$.
\end{itemize}
\end{definition}

A tree over $D$ with root $+p$ indicates that $p$ is
defeasibly derivable from $D$; a tree over $D$ with root $-p$ indicates that $p$ is defeasibly refuted. In order to count as a valid proof in NDL or ADL, the nodes of the tree
 must satisfy certain conditions.  We discuss the conditions for NDL first.

\begin{definition}\label{NDLProofB} An argument tree $\tau$ over defeasible theory $D$ is an NDL-\emph{proof} for $D$ iff for each node $n$ of $\tau$, one of the following
obtains.
    \begin{itemize}

    \item[1.]  $label(n) =  +p$ and either

\begin{itemize}

    \item[a.]  there is an $r \in R_s[p]$ such that $body(r)$ succeeds at $n$, or

    \item[b.]  there is an $r \in R_d[p]$ such that

    \begin{itemize}
          \item[i.] $body(r)$ succeeds at $n$, and
          \item[ii.]  for all $c \in C[p]$ there is a $q \in c-\{p\}$
     such that for all $s \in R[q]$, either $body(s)$ fails at $n$ or else $s \prec
     r$.

\end{itemize}

\end{itemize}
         \item[2.]  $label(n) = -p$ and
\begin{itemize}

    \item[a.]  for all $r \in R_s[p]$, $body(r)$ fails at $n$, and

    \item[b.]  for all $r \in R_d[p]$, either
    \begin{itemize}
          \item[i.] $body(r)$ fails at $n$, or
          \item[ii.] there is a $c \in C[p]$ such that for all $q \in c-\{p\}$,
          there is a $s \in R[q]$ such that $body(s)$ succeeds at $n$  and $s \nprec
          r$.
\end{itemize}
\end{itemize}
   \item[3.]  $label(n) = -p$ and $n$ has an ancestor $m$ in $\tau$  with $label(m) = -p$,
   and all nodes between $n$ and $m$ are negative.
\end{itemize}
\end{definition}

\begin{definition}\label{NDLProofConsequences}Let  $D$ be a defeasible theory and $\tau$  an NDL-proof for $D$.

\begin{enumerate}
\item
$\tau$ is an NDL-\emph{proof of} $p$ \emph{in} $D$ iff $\tau$ is an NDL-proof for $D$, $p \in Lit(D)$, and the root node of $\tau$ is labeled $+p$. If such a proof exists, then $D \dmodels_{NDL} p$.

\item
$\tau$ is an NDL-\emph{refutation of} $p$ \emph{in} $D$ iff $\tau$ is an NDL-proof for $D$, $p \in Lit(D)$, and the root node of $\tau$ is labeled $-p$.
If such a refutation exists, then $D \drefutes_{NDL} p$.
\end{enumerate}
\end{definition}

The third condition in Definition \ref{NDLProofB} is called \emph{failure-by-looping}, and it prevents a literal from being derived using a circular argument. According to the condition, the nodes between $n$ and $m$ must all be negative. This ensures that literals are not simultaneously provable and refutable.
It is failure--by--looping that requires the proofs to be trees rather than linear sequences of literals.

NDL is an ambiguity blocking logic. Returning to Example \ref{ambiex}, the conclusions are
$D \drefutes_{NDL} p$, $D \drefutes_{NDL} \neg p$, $D \drefutes_{NDL} \neg q$, and $D \dmodels_{NDL} q$.
NDL can be modified, however, in a simple way to make it propagate ambiguity---yielding ADL. In ADL, a defeasible rule $r$ can only be defeated by a conflicting set of rules that are strict or else of {higher priority} (in NDL, rules simply not inferior to $r$ can be used).
 The modification to the proof system is shown in Definition \ref{ADLProofB}.
 Proofs and refutations in ADL are otherwise defined as they are in NDL.

\begin{definition}\label{ADLProofB}
An argument tree $\tau$ for $D$ is an ADL-\emph{proof for} $D$ iff  each
node $n$ of $\tau$ satisfies conditions 1, 2.a, 2.b.i, or 3 of Definition \ref{NDLProofB}, or else the modified condition 2.b.ii below:

\begin{center}
\begin{minipage}{0.9\textwidth}
\begin{itemize}
          \item[2.b.ii.] there is a $c \in C[p]$ such that for all $q \in c-\{p\}$,
          there is a $s\in R[q]$ such that $body(s)$ succeeds at $n$  and $s$ is strict or else $r \prec
          s$.
\end{itemize}
\end{minipage}
\end{center}
\end{definition}

Proofs in ADL and NDL are finite trees, and so must work with finite sets of literals. The fixpoint semantics for NDL and ADL can work with infinite sets, however, and because of this the proof systems cannot be complete with respect to their counterpart semantics. Nevertheless, the proof systems are sound with respect to the semantics, and for \emph{locally finite theories} they are also complete.

\begin{definition}\label{Dep} Let $D$ be a defeasible theory and  $p \in Lit(D)$.

\begin{enumerate}
\item $Dep_D(p)$ is the smallest set such that (i)
    $p\in Dep_D(p)$; (ii) for each $q \in Dep_D(p)$, if  $c \in C[q]$,  then $c \subseteq Dep_D(p)$; and (iii)
    for each $q \in Dep_D(p)$, if  $r \in R[q]$,  then $body(r) \subseteq Dep_D(p)$.
\item Literal $p$ is \emph{locally finite in} $D$ iff $Dep_D (p)$ is finite.
 \item $D$ is \emph{locally finite}  iff each literal of $Lit(D)$ is locally finite in $D$.
\end{enumerate}
\end{definition}

\begin{proposition} \cite{maier2009} If $D$ is a defeasible theory, $p \in Lit(D)$, and $L$ one of NDL or ADL, %then

\begin{enumerate}
\item  $D \dmodels_L p$ implies $D \dsmodels_L p$, and    $D \drefutes_L p$, implies  $D \dsrefutes_L p$.
\item If $p$ is locally finite in $D$,  $D \dsmodels_L p$ implies $D \dmodels_L p$ and
 $D \dsrefutes_L p$ implies $D \drefutes_L p$.
 \end{enumerate}
\end{proposition}

\section{Eliminating Priorities and Defeaters}\label{elimDefeaters}

If $D=\langle R_D,C_D,\prec_D\rangle$ is a defeasible theory such that $C_D[p]$ is finite for all $p \in Lit(D)$, then $D$ can be translated into an equivalent theory $E=\langle R_E,C_E,\prec_E\rangle$, called the \emph{defeater-- and priority--free form} of $D$, in which $R_{E,u} = \varnothing$, $\prec_{E} = \varnothing$, and $C_E$ is minimal. Specifically, $E$ is the smallest theory such that
the following hold (in the following, $A \dashrightarrow p$  stands for an arbitrary rule):

\begin{enumerate}
\item If rule $r$: $A \dashrightarrow p$ is in $R_D$,
the  rules

\begin{enumerate}

\item[$r'$:] $A \rightarrow su(r)$

\item[$r''$:] $\{su(r)\} \dashrightarrow fi(r)$
\end{enumerate}
appear in $R_E$, with $r''$ strict (defeasible) if $r$ is strict (defeasible or a defeater).

\item If $r$: $A \dashrightarrow p$ is strict or defeasible, the
following also occurs in $R_E$:

\begin{enumerate}
\item[$r'''$:] $\{fi(r)\} \rightarrow p$.
\end{enumerate}

\item Let $c = \{q_1, \ldots, q_n, p\} \in C_D[p]$ be a conflict set,
$r \in R_{D,d}[p]$ and $s_1,  \ldots, s_n$ rules such that $s_i \in
R_D[q_i]$ and  $s_i \nprec_D r$. The following rule appears in $R_E$:

\begin{center}$\{su(s_1),  \ldots, su(s_n)\} \dashrightarrow \neg fi(r)$.\end{center}
It is strict if $s_i \in R_s$ or $r \prec_D s_i$ for all $s_i$, and defeasible otherwise.
\end{enumerate}

The rules of $E$ explicitly encode
when a rule $r$ of $D$  is supported and when it may fire. In item 3, only rules $s_i$ that could defeat $r$ are used. For any conflict set $c \in C_D[head(r)]$,
let $trans(c,r)$ denote the set of rules for $\neg fi(r)$ created
from $c$.  Importantly, the conflict sets of $E$ are minimal.  If $D$ itself uses minimal conflict sets, then if $r \in R_{D,sd}[p]$ and $s
\in R_{}[\neg p]$, condition 3 above reduces to $\{su(s)\} \rightarrow
\neg fi(r)$ if $s$ is strict or $r \prec_D s$, and to $\{su(s)\} \Rightarrow \neg fi(r)$ if  $s \nprec_D r$.

\begin{proposition}Let $D$ be a defeasible theory such that for all $p \in Lit(D)$, $C_D[p]$ is finite, and let $E$ be the defeater-- and priority--free form of $D$.
If $p \in \itrue{D,WF}$ ($p \in \ifalse{D,WF}$),  then $p \in \itrue{E,WF}$ ($p \in \ifalse{E,WF}$).
\end{proposition}

\begin{proof*}The proof proceeds by induction on the sequence $\tva{D,0}$, $\tva{D,1}$, $\ldots$, and   shows that for all $\ordJ \geq 0$, if $p \in \itrue{D,\ordJ}$  then $p \in \itrue{E,WF}$, and if $p \in
\ifalse{D,\ordJ}$ then  $p \in
\ifalse{E,WF}$. The claim holds trivially for $\ordJ=0$.  Suppose that it holds for all $\ordJ < \ordK$. We can assume without loss of generality that $\ordK$ is a successor ordinal.

\begin{enumerate}
\item  Suppose  $p \in \itrue{D,\ordK}$.  Then there is an $r \in R_{D}[p]$ such that
$body(r) \subseteq \itrue{D,\ordK-1}$ and either (1) $r \in R_{D,s}$ or else
(2) $r \in R_{D,d}$ and for each $c \in C_D[p]$, there is a $q \in c -
\{p\}$ and for all $s \in R_D[q]$, $body(s) \cap \ifalse{D,\ordK-1} \neq \varnothing$ or $s
\prec_D r$. By the inductive hypothesis, $body(r) \subseteq \itrue{E,WF}$, and so  $su(r) \in
\itrue{E,WF}$. If (1) holds, $r''$ is strict, and so $p \in \itrue{E,WF}$.

So suppose (2) obtains. Then $r''$ is defeasible. By the inductive hypothesis, for every $c \in C_D[p]$, there is a $q \in c - \{p\}$ such that
for every rule $s \in R_D[q]$, either (i) $body(s) \cap \ifalse{E,WF} \neq \varnothing$, or else (ii)
$s \prec_{D} r$. In other words, if $s \nprec_{D} r$ then $body(s) \cap \ifalse{E,WF} \neq \varnothing$ and so
$su(s) \in \ifalse{E,WF}$.
As such, for every rule $t \in trans(c,r)$, $body(t) \cap \ifalse{E,WF}
\neq \varnothing$. Generalizing on $c$, for every rule $t
\in R_{E}[\neg fi(r)]$, $body(t) \cap \ifalse{E,WF}  \neq
\varnothing$. By definition of $T_{E}$ and $\tva{E,WF}$,  both $fi(r) \in \itrue{E,WF}$ and $p \in \itrue{E,WF}$.

  \item Now suppose that $p \in \ifalse{D,\ordK}$. %Note that $\ifalse{D,\ordK}$ is an unfounded set for ADL (NDL) with regard to $D$ and any interpretation $\tva{D, \alpha}$ on the sequence.
Let $X = \ifalse{D,\ordK} \cup \{fi(r)|$ $r \in R_D[p]$ and $p \in
\ifalse{D,\ordK}\} \cup \{su(r),fi(r)|$ $body(r) \cap
\ifalse{D,\ordK} \neq \varnothing\}$.  We will show that $X$ is unfounded
\emph{wrt} $E$ and $\tva{E, WF}$.  Note that there are
three types of literal in $X$: Those
in $\ifalse{D,\ordK}$, and those of the form $fi(r)$ or  $su(r)$.
Regarding the first two types,
if $p\in \ifalse{D,\ordK}$, then  (by definition of $X$) $fi(r) \in X$ for all
rules $\{fi(r)\} \rightarrow p  \in R_E[p]$.
This exhausts all rules in $E$ for $p$. If $su(r) \in X$, then by
definition of $X$, $body(r) \cap \ifalse{D,\ordK} \neq \varnothing$.

Regarding the third type,
suppose $fi(r) \in X$. Recall   $r \in R_D[p]$ for some $p$, and $r'$ is the only rule in $E$ for
$fi(r)$.  If $body(r)
 \cap \ifalse{D,\ordK} \neq \varnothing$, then $su(r) \in X$.  If $body(r)
 \cap \ifalse{D,\ordK} = \varnothing$, then (by definition of $X$)  $p \in
 \ifalse{D,\ordK}$.  Since $p \in \ifalse{D,\ordK}$ and $body(r)
 \cap \ifalse{D,\ordK} = \varnothing$, $r$ (and $r'$) must be defeasible,
and there must be a $c \in C_D[p]$
such that for all $q_i \in c-\{p\}$, there is a rule $s_i \in R_D[q_i]$ such
that $body(s_i)\in \itrue{D,\ordK-1}$ and $r \prec_{D} s_i$ or $s_i$ is strict (for NDL, $s_i
\nprec_{D} r$). This implies that the rule $t:$ $\{su(s_1), su(s_2), \ldots,
su(s_\ordK)\} \rightarrow \neg fi(r)$ appears in $E$ (for NDL
$t$ is defeasible). By the inductive hypothesis, for each $s_i$,
$body(s_i) \in \itrue{E,WF}$, and so $su(s_i) \in \itrue{E,WF}$.

$X$ is thus unfounded \emph{wrt} $E$ and $\tva{E,WF}$ (i.e., $X \subseteq \ifalse{E,WF}$), and so $p \in \ifalse{E,WF}$. $\square$

\end{enumerate}
\end{proof*}

\begin{proposition}Let $D$ be a defeasible theory such that for all $p \in Lit(D)$, $C_D[p]$ is finite, and let $E$ be the defeater-- and priority--free form of $D$.
 If $p \in \itrue{E,WF}$ ($p \in \ifalse{E,WF}$), then $p \in \itrue{D,WF}$ ($p \in \ifalse{D,WF}$).
\end{proposition}

\begin{proof*}
The proof proceeds by induction on the sequence $\tva{E,0}$, $\tva{E,1}$, $\ldots$, and
shows that for all $\ordJ \geq 0$, if $p \in \itrue{E,\ordJ}$  then $p \in \itrue{D,WF}$, and if $p \in
\ifalse{E,\ordJ}$ then  $p \in
\ifalse{D,WF}$. The claim holds trivially for $\ordJ=0$.  Suppose that it holds for all $\ordJ < \ordK$.
 We can again assume without loss of generality that $\ordK$ is a successor ordinal.

\begin{enumerate}
\item Suppose $p \in \itrue{E,\ordK}$. Then there is an  $r \in R_D[p]$, matching rules $r',r'',r''' \in R_E$, and a least $\ordG < \ordK$  such that $\{fi(r), su(r)\}\cup body(r) \subseteq
\itrue{E,\ordG}$. By the inductive hypothesis, $body(r) \subseteq
\itrue{D,WF}$. If $r \in R_{D,s}[p]$, then clearly $p \in \itrue{D,WF}$.

So suppose  $r \in
R_d[p]$.   For every $t \in R_E[\neg fi(r)]$, there is an
$su(s) \in body(t)$ such that $su(s)\in
\ifalse{E,\ordG}$. This implies that $body(s) \cap \ifalse{E,\ordG} \neq
\varnothing$. By the inductive hypothesis, $body(s) \cap \ifalse{D,WF} \neq
\varnothing$.   % Then for each $t \in trans(c,r)$ there is a $su(s) \in body(t)$ such that $body(s) \cap \ifalse{D,WF} \neq \varnothing$.
Given this (and the definition of $trans(c,r)$), for any  $c\in C_D[p]$, there must be a $q \in c - \{p\}$ such that for each $s \in R_D[q]$, $body(s) \cap
\ifalse{D,WF} \neq \varnothing$ or else $s \prec_{D} r$.
By definition of $T_D$ and $\tva{D,WF}$, $p \in \itrue{D,WF}$.

\item Now suppose $p \in \ifalse{E,\ordK}$ and let $a$ be any literal in $Lit(D) \cap \ifalse{E,\ordK}$.  Then for each rule $r'''$:
$\{fi(r)\} \rightarrow a$, $fi(r) \in \ifalse{E,\ordK}$ and for each
rule $r''$: $\{su(r)\} \dashrightarrow fi(r)$, either (1)
$su(r) \in \ifalse{E,\ordK}$,
or (2) $r''$ is defeasible and there is a $t:$ $su(s_1),
 \ldots, su(s_n) \dashrightarrow \neg fi(r)$
such that $body(t) \subseteq \itrue{E,\ordK-1}$ and $t$ is strict (for NDL,
$t \nprec_E r''$). If (1) it follows that $body(r) \subseteq
\ifalse{E,\ordK}$. If (2) then for each $su(s_i) \in body(t)$, $body(s_i) \subseteq \itrue{E, \ordC}$ for some $\ordC < \ordK$ and by the inductive hypothesis $body(s_i) \subseteq \itrue{D, WF}$.
Given the construction of rules such as $t$, there exists a $c \in C_D[a]$ such that for all $q \in c-\{a\}$ there is a $s \in R_D[q]$ such that
$body(s) \subseteq \itrue{D,WF}$ and $r \prec_D s$ or $s$
strict (for NDL, $s \nprec_D r$).

Generalizing on $a$,
$\ifalse{E,\ordK}$ is unfounded \emph{wrt} $D$ and $\itrue{D,WF}$, and
so $p \in \ifalse{D,WF}$. $\square$
\end{enumerate}
\end{proof*}

\end{document}